\documentclass[lettersize,journal]{IEEEtran}
\usepackage{amsmath,amsfonts}
\usepackage{algorithmic}
\usepackage{algorithm}
\usepackage{array}
\usepackage[caption=false,font=normalsize,labelfont=sf,textfont=sf]{subfig}
\usepackage{textcomp}
\usepackage{stfloats}
\usepackage{url}
\usepackage{verbatim}
\usepackage{graphicx}
\usepackage{cite}
\usepackage{bm}
\usepackage{makecell}

\usepackage{amsfonts}  
\usepackage{algorithm}
\usepackage{algorithmic}
\usepackage{siunitx}   
\usepackage{upgreek}   
\usepackage{xcolor}    
\usepackage{booktabs}       
\usepackage{longtable}
\usepackage{threeparttable}
\usepackage{multirow}
\usepackage{multicol}
\usepackage{times}
\usepackage{hyperref}
\usepackage{enumitem}

\usepackage{amsthm}

\newtheorem{lemma}{Lemma}

\newtheorem{prop}{Proposition}

\def\x{\boldsymbol{x}}

\def\u{\mathbf{u}}
\def\W{\mathbf{W}}

\def\X{\mathbf{X}}

\newcommand{\fl}[1] {\textcolor{black}{#1}}

\begin{document}
	
\title{\huge\bf On Expressivity and Trainability of Quadratic Networks}

\author{Feng-Lei Fan$^{1}$, \textit{Member, IEEE}, Mengzhou Li$^{2}$, \textit{Student Member, IEEE}, Fei Wang$^{1*}$, Rongjie Lai$^{3*}$, \\
Ge Wang$^{2*}$, \textit{Fellow, IEEE} 
\thanks{*These authors are co-corresponding authors.}
\thanks{$^{1}$Drs. Feng-Lei Fan (fanf2@rpi.edu) and Fei Wang (few2001@med.cornell.edu) are with Department of Population Health Sciences, Weill Cornell Medicine, Cornell University, New York, NY, USA }
\thanks{$^{2}$Mengzhou Li and Dr. Ge Wang (wangg62@rpi.edu) are with AI-based X-ray Imaging System (AXIS) Lab, Biomedical Imaging Center, Rensselaer Polytechnic Institute, Troy 12180, NY, USA}
\thanks{$^{3}$Dr. Rongjie Lai (lair@rpi.edu) is with Department of Biomedical Engineering, Rensselaer Polytechnic Institute, Troy, NY 12180, USA}        %
}

\markboth{Journal of \LaTeX\ Class Files,~Vol.~14, No.~8, August~2021}%
{Shell \MakeLowercase{\textit{et al.}}: A Sample Article Using IEEEtran.cls for IEEE Journals}

	
\maketitle
	
\begin{abstract}
Inspired by the diversity of biological neurons, quadratic artificial neurons can play an important role in deep learning models. The type of quadratic neurons of our interest replaces the inner-product operation in the conventional neuron with a quadratic function. Despite promising results so far achieved by networks of quadratic neurons, there are important issues not well addressed. Theoretically, the superior expressivity of a quadratic network over either a conventional network or a conventional network via quadratic activation
is not fully elucidated, which makes the use of quadratic networks not well grounded. Practically, although a quadratic network can be trained via generic backpropagation, it can be subject to a higher risk of collapse than the conventional counterpart. To address these issues, we first apply the spline theory and a measure from algebraic geometry to give two theorems that demonstrate better model expressivity of a quadratic network than the conventional counterpart with or without quadratic activation. Then, we propose an effective training strategy referred to as ReLinear to stabilize the training process of a quadratic network, thereby unleashing the full potential in its associated machine learning tasks. Comprehensive experiments on popular datasets are performed to support our
findings and confirm the performance of quadratic deep learning. We have shared our code in \url{https://github.com/FengleiFan/ReLinear}.
\end{abstract}
	
	
\section{INTRODUCTION}

\IEEEPARstart{I}n recent years, a plethora of deep artificial neural networks have been developed with impressive successes in many mission-critical tasks \cite{fuchs2021super, you2019ct}. However, up to date the design of these networks focuses on architectures, such as shortcut connections \cite{he2016deep, fan2018sparse}. Indeed, neural architecture search \cite{liu2018progressive} is to find networks of similar topological types. Almost exclusively, the mainstream network models are constructed with neurons of the same type, which is composed of two parts: inner combination and nonlinear activation (We refer to such a neuron as a conventional neuron and a network made of these neurons as a conventional network hereafter). Despite that a conventional network does simulate certain important aspects of a biological neural network such as a hierarchical representation \cite{lecun2015deep}, attention mechanism \cite{air}, and so on, a conventional network and a biological neural system are fundamentally different in terms of neuronal diversity and complexity. In particular, a biological neural system coordinates numerous types of neurons which contribute to all kinds of intellectual behaviors \cite{thivierge2008neural}. Considering that an artificial network is invented to mimic the biological neural system, the essential role of neuronal diversity should be taken into account in deep learning research. 

Along this direction, the so-called quadratic neurons \cite{fan2019quadratic} were recently proposed, which replace the inner product in a conventional neuron with a quadratic operation (Hereafter, we call a neural network made of quadratic neurons as a quadratic neural network, QNN). A single quadratic neuron can implement XOR logic operation, which is not possible for an individual conventional neuron. Therefore, there must exist a family of functions that can be represented by a quadratic neuron, but cannot be represented by a conventional neuron, suggesting a high expressivity of quadratic neurons. Furthermore, the superior expressivity of quadratic networks over conventional networks is confirmed by a theorem that given the same structure there exists a class of functions that can be expressed by quadratic networks with a polynomial number of neurons, and can only be expressed by conventional networks with an exponential number of neurons \cite{fan2020universal}. In addition, a quadratic autoencoder was developed for CT image denoising and produced desirable denoising performance \cite{fan2019quadratic}. 

In spite of the promising progress achieved by quadratic networks, there are still important theoretical and practical issues unaddressed satisfactorily. First of all, the superiority of quadratic networks in the representation power can be analyzed in a generic way instead of just showing the superiority in the case of a special class of functions. Particularly, we are interested in comparing a quadratic network with both a conventional network and a conventional network with quadratic activation \cite{du2018power,mannelli2020optimization,xu2021robust}. Also, the training process of a quadratic network is subject to a higher risk of collapse than the conventional counterpart. Specifically, given a quadratic network with $L$ layers, its output function is a polynomial of $2^L$ degrees. Such a degree of freedom may dramatically amplify the slight change in the input, causing oscillation in the training process. We can use a univariate example to illustrate the specific reasons that cause training curve oscillation. Let us build a quadratic and a conventional network for the univariate input. The network structure is 1-10-10-10-10-1, and the activation function is ReLU. All parameters of the quadratic and conventional networks are initialized with the normal distribution $\mathcal{N}(0,\sigma^2)$. Figure \ref{Fig_example_oscillation} shows the outputs of two networks when $\sigma=0.01,0.02,0.03$, respectively. It can be seen that a slight change in $\sigma$ can result in a huge magnitude difference in the output of the quadratic network, while such a change only affects the output of the conventional network moderately. Therefore, it is essential to derive an effective strategy to stablize the training process of a quadratic network.

\begin{figure}[htb]
\center{\includegraphics[width=0.8\linewidth] {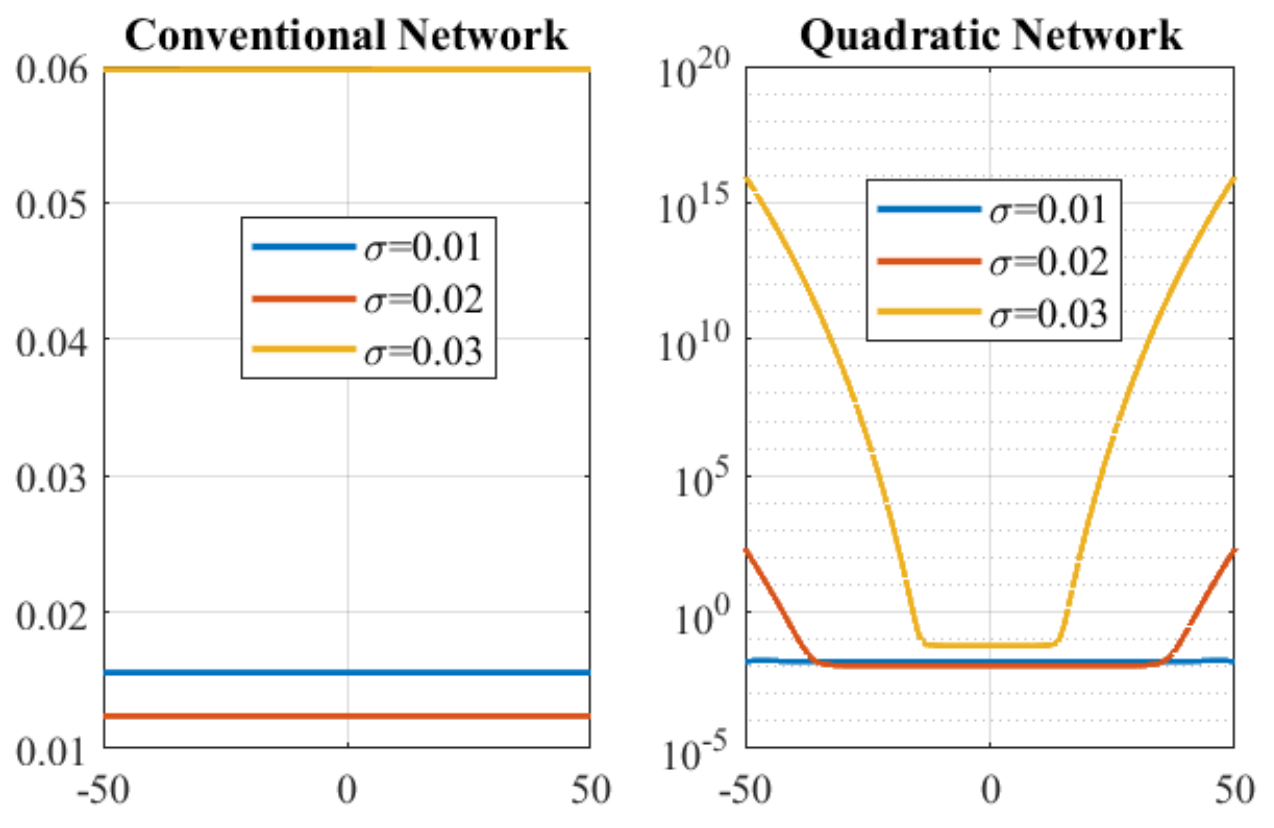}}
\caption{Illustration of why oscillation occurs in quadratic networks. The network structure is 1-10-10-10-10-1, and the activation function is ReLU. All parameters of the quadratic and conventional networks are initialized with the normal distribution $\mathcal{N}(0,\sigma^2)$, where $\sigma$ is set to $0.01,0.02,0.03$, respectively.}
\vspace{-0.5cm}
\label{Fig_example_oscillation}
\end{figure}

To address the above issues, here we first present two theorems to reveal the superiority of a quadratic network in terms of model expressivity over either the conventional network or the conventional network with quadratic activation. The first theorem utilizes spline theory to compare the model expressivity of a quadratic network with that of a conventional network. Suppose that a ReLU activation is used, a conventional network outputs a piecewise linear function, and a quadratic network defines a piecewise polynomial function. According to spline theory, the approximation with a piecewise polynomial function is substantially more accurate than that with piecewise linear functions. Correspondingly, a quadratic network enjoys a better approximation accuracy. The other theorem is based on a measure in algebraic geometry to show that a quadratic network is more expressive than a conventional network with quadratic activation, which suggests that a conventional network with quadratic activation is not optimal to leverage quadratic mapping for deep learning. 

 \begin{figure}[htb]
\center{\includegraphics[width=0.9\linewidth] {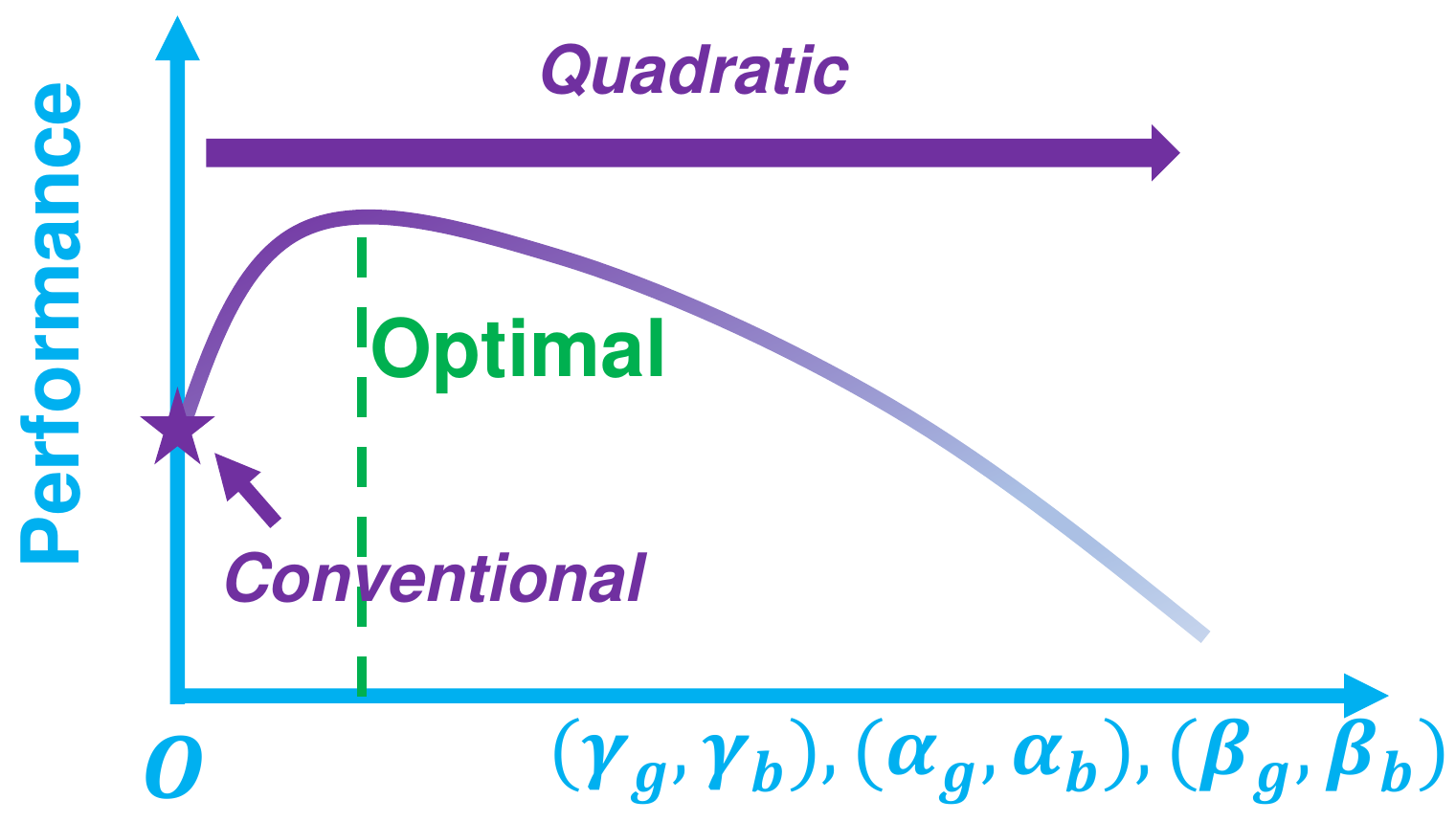}}
\caption{The performance of a quadratic network trained using the proposed \textit{ReLinear} method, with an observed improvement than the conventional network of the same structure. $(\gamma_g,\gamma_b)$, $(\alpha_g,\alpha_b)$, and $(\beta_g,\beta_b)$ are hyperparameters of \textit{ReLinear}. As these hyperparameters increases from $0$, the trained model transits from the conventional model to the quadratic, and the model's performance reaches the optimality.} 
\label{fig:guaranteed}
\end{figure}

To unleash the full potential of a quadratic network, as Figure \ref{fig:guaranteed} shows, we propose a novel training strategy referred to as \textit{ReLinear} (Referenced Linear initialization) to stabilize the training process of a quadratic network, with which we let each quadratic neuron evolve from a conventional neuron to a quadratic neuron gradually. Moreover, regularization is imposed to control nonlinear terms of a quadratic neuron to favor a low-order polynomial fitting. As a result, not only the training process is stabilized but also a quadratic network so trained can yield a performance gain compared to the conventional network of the same structure. Furthermore, encouraged by the success of ReZero (residual with zero initialization) in training a residual network \cite{bachlechner2020rezero}, we merge our ReLinear and ReZero to train a quadratic residual network progressively. Finally, we evaluate the ReLinear strategy in comprehensive experiments on well-known data sets.

\textbf{Main Contributions.} In this paper, we present two theorems to demonstrate the superiority of quadratic networks in functional expressivity. Our results show not only that a quadratic network is powerful in the deep learning armory but also that a network with quadratic activation is sub-optimal. Of great practical importance, we propose a novel training strategy to optimize quadratic networks. Finally, we conduct comprehensive experiments to demonstrate that the quadratic network trained with the proposed training strategy can perform competitively on well-known datasets.

\section{RELATED WORK}
\textbf{Polynomial networks} were investigated in the late 90s. The idea of polynomial networks can be traced back to the Group Method of Data Handling (GMDH \cite{ivakhnenko1971polynomial}), which learns gradually a complicated feature extractor:
\begin{equation}
\begin{aligned}
     Y(\x_1,\cdots,\x_n) =& a_0 + \sum_i^n a_i \x_i + \sum_i^n \sum_j^n a_{ij} \x_i \x_j \\
    & + \sum_i^n \sum_j^n \sum_k^n a_{ijk} \x_i \x_j \x_k + \cdots ,
\end{aligned}
\end{equation}
where $\x_i$ is the $i$-th input variable, and $a_i, a_{ij}, a_{ijk},\cdots$ are coefficients. Usually, this model is terminated at the second-order terms to avoid nonlinear explosion for high-dimensional inputs. The GMDH is thought of as one of the first deep learning models in the survey paper \cite{schmidhuber2015deep}. Furthermore, based on GMDH, the so-called higher-order unit was defined in \cite{poggio1975optimal,giles1987learning, lippmann1989pattern} whose output is given by 
\begin{equation}
    y=\sigma(Y(\x_1,\cdots,\x_n)), 
\end{equation}
where $\sigma(\cdot)$ is a nonlinear activation function. To maintain the power of high-order units while reducing the number of weights in high-order terms, Shin \textit{et al.} reported the pi-sigma network \cite{shin1991pi}, which is formulated as
\begin{equation}
    h_{ji} = \sum_k \omega_{kji} \x_k + \theta_{ji} \ \ \mathrm{and} \ \ y_i = \sigma(\prod_j h_{ji}), 
\label{HOunits}    
\end{equation}
where $h_{ji}$ is the output of the $j$-th sigma unit for the $i$-th output element $y_i$, and $\omega_{kji}$ is the weight of the $j$-th sigma unit associated with the input element $\x_k$. A pi-sigma network is intrinsically a shallow quadratic network. Along this direction, \cite{milenkovic1996annealing} removed all cubic and higher-order terms and proposed to use the annealing technique to find optimal quadratic terms.

Recently, higher-order units were revisited \cite{zoumpourlis2017non, tsapanos2018neurons, chrysos2021deep, livni2014computational, krotov2018dense}. In \cite{redlapalli2003development, zoumpourlis2017non, jiang2020nonlinear, mantini2021cqnn}, a quadratic convolutional filter $\sigma(\x^\top W \x)$ whose  complexity is $\mathcal{O}(n^2)$ was proposed to replace the linear filter, while in \cite{tsapanos2018neurons} a parabolic neuron: 
$\sigma\Big((\x^\top\mathbf{w}^{r}+b^{r})(\x^\top\mathbf{w}^{g}+b^{g})\Big)$ was proposed for deep learning, which is a special case of our neuron without the power term. 
Bu \textit{et al.} \cite{bu2021quadratic} proposed a quadratic neuron $\sigma\Big((\x^\top\mathbf{w}_1)(\x^\top\mathbf{w}_2)+\x^\top\mathbf{w}_3\Big)$, which is a special case of our quadratic design when excluding the bias terms. Xu \textit{et al.}'s quadratic neuron design \cite{xu2022quadralib} is the same as \cite{bu2021quadratic}. \cite{goyal2020improved} proposed a quadratic neuron: $\sigma\Big((\x\odot\x)^\top\mathbf{w}^{b}+c\Big)$, which is also a special case of our neuron excluding the interaction term. 
In \cite{chrysos2021deep}, the higher-order units as described by Eq. \eqref{HOunits} were embedded into a deep network to reduce the complexity of the individual unit via tensor decomposition and factor sharing. Such a network achieved cutting-edge performance on several tasks. Compared to \cite{chrysos2021deep}, our group proposed a simplified quadratic neuron with $\mathcal{O}(3n)$ parameters and argued that more complicated neurons are not necessary based on the fundamental theorem of algebra \cite{remmert1991fundamental}. Interestingly, when only the first and second-order terms are kept, and the rank is set to two in tensor decomposition, the formulation of the polynomial neuron in \cite{chrysos2021deep} is $\sigma\Big((\x^\top\mathbf{w}^{r}+b^{r})(\x^\top\mathbf{w}^{g}+b^{g})+\x^\top\mathbf{w}^{r}+b^{r}\Big)=\sigma\Big((\x^\top\mathbf{w}^{r}+b^{r})(\x^\top\mathbf{w}^{g}+b^{g}+1)\Big)$. In this regard, the polynomial neuron in \cite{chrysos2021deep} is also a special case of our quadratic neuron by setting $b^{r}=b^{r}+1$, $\mathbf{w}^{b}=0$, and $c=0$.

On the other hand, neurons with polynomial activation \cite{livni2014computational, krotov2018dense} are also relevant. However, the polynomially activated neurons are essentially different from polynomial neurons. In the networks with a polynomial activation, their neurons are still characterized by a piece-wise linear decision boundary, while in the latter case, a polynomial decision boundary is implied that can truly extract nonlinear features. Kileel \textit{et al.} \cite{kileel2019expressive} found that a network with polynomial activation is an algebraic variety, and proposed the dimension of the algebraic variety to measure the representation power of such a network.
 
\textbf{Development of quadratic networks.} In a theoretical perspective, the representation capability of quadratic networks was partially addressed in the analyses on the role of multiplicative operations in a network \cite{jayakumar2019multiplicative}, where the incorporation of multiplicative interactions can strictly enlarge the hypothesis space of a feedforward neural network. Fan \textit{et al.} \cite{fan2020universal} showed that a quadratic network can approximate a complicated radial function with a more compact structure than a conventional model. More results are on applications of quadratic networks. For example, Nguyen \textit{et al.} \cite{nguyen2019deep} applied quadratic networks to predict the compressive strength of foamed concrete. Bu \textit{et al.} \cite{bu2021quadratic} applied a quadratic network to solve forward and inverse problems in partial different equations (PDEs).

\section{EXPRESSIVITY OF QUADRATIC NETWORKS}

Given an input $\x\in\mathbb{R}^n$, a quadratic neuron of our interest is characterized as $\sigma(q(\x))$, which is
\begin{equation}
\begin{aligned}
&\sigma\Big((\sum_{i=1}^{n} w_{i}^r x_i +b^r)(\sum_{i=1}^{n} w_{i}^g x_i +b^g) + \sum_{i=1}^{n} w_{i}^b x_{i}^2+c \Big) \\
=&\sigma\Big((\x^\top\mathbf{w}^{r}+b^{r})(\x^\top\mathbf{w}^{g}+b^{g})+(\x\odot\x)^\top\mathbf{w}^{b}+c\Big),
\end{aligned}
\end{equation}
where $\sigma(\cdot)$ is a nonlinear activation function (hereafter, we use $\sigma(\cdot)$ to denote ReLU), $\odot$ denotes the Hadamard product, $\mathbf{w}^r,\mathbf{w}^g, \mathbf{w}^b\in\mathbb{R}^n$, and $b^r, b^g, c\in\mathbb{R}$ are biases. We use the superscription $r, g, b$ for better distinguishment. For a univariate input, we have
\begin{equation}
q(x)=(w^{r}x +b^r)(w^{g}x +b^g) + w^{b}x^2+c.
\end{equation}

A network with higher expressivity means that this network can either express more functions or express the same function more accurately. In this section, we show the enhanced expressivity of our quadratic network relative to either a conventional network or a conventional network with quadratic activation. For comparison with a conventional network, we note that in spline theory, a polynomial spline has a significantly more accurate approximation power than the linear spline. Since a quadratic network can express a polynomial spline and a conventional network corresponds to a linear spline,  a quadratic network is naturally more powerful than a conventional network. As far as conventional networks with quadratic activation are concerned, we leverage the \textit{dimension of algebraic variety} as the model expressivity measure defined in \cite{kileel2019expressive} to demonstrate that our quadratic network has a higher dimension of algebraic variety, which suggests that a quadratic network is more expressive than a conventional network with quadratic activation.

\subsection{Spline Theory}
Let $f$ be a function in $C^{n+1}[a,b]$ and $p$ be a polynomial to interpolate the function $f$ according to $n+1$ distinct points $x_0, x_1, \cdots,x_{n} \in [a,b]$. Then, for any $x \in [a,b]$, there exists a point $\xi_x \in [a,b]$ such that
\begin{equation}
    f(x) - p(x) = \frac{1}{(n+1)!} f^{(n+1)}(\xi_x) \prod_{i=0}^{n} (x-x_i). 
\end{equation}
For some function such as the Runge function $R(x)=1/(1+16x^2)$, as the degree of the polynomial $p(x)$ increases, the interpolation error goes to infinity, \textit{i.e.},
\begin{equation}
    \underset{n \to \infty }{\lim} \Big( \underset{-1\leq x \leq 1}{\max} |f(x)-p_n(x)| \Big) = \infty.
\end{equation}
This bad behavior is referred to as the Runge phenomenon \cite{boyd1992defeating}, which is explained by two reasons: 1) As $n$ grows, the magnitude of the $n$-th derivative increases; and 2) large distances between the points makes $\displaystyle\prod_{i=0}^{n} (x-x_i)$ huge.  

To overcome the Runge phenomenon when a high-order polynomial is involved for interpolation, the polynomial spline \cite{birman1967piecewise} is invented to partition a function into pieces and fit a low-order polynomial to a small subset of points in each piece. Given a set of instances $\{(x_i, f(x_i))\}_{i=0}^n$ from the function $f(x)$, a polynomial spline is generically formulated as follows:

\begin{equation}
       S(x) = \begin{cases}
        s_0 (x)    , \ \ \ \ \ \ \ \ x_0 \leq x < x_1 \\
        s_1(x)     , \ \ \ \ \ \ \ \ x_1 \leq x < x_2 \\
        \ \ \ \ \ \ \ \ \ \ \ \ \ \ \ \ \ \ \ \ \ \ \  \vdots \\
        s_{n-1}(x)  , \ \ \ \ x_{n-1} \leq x \leq x_{n} 
        \end{cases},
\label{eqn:polynomial_spline}        
\end{equation}
where $S(x)$ is a polynomial spline of the order $2m-1$ (w.l.o.g., we consider the odd-degree polynomial splines), satisfying that (1) $s_i(x_{i+1})=s_{i+1}(x_{i+1})= f(x_{i+1})$ for $i=0,1,\cdots,n$; (2) $s_i^{(2k)}(x_{i+1})=s_{i+1}^{(2k)}(x_{i+1})=f^{(2k)}(x_{i+1}), i=0,1\cdots,n-2, k=1,2,\cdots,m-1$. The simplest piecewise polynomial is a piecewise linear function. However, a piecewise linear interpolation is generically inferior to a piecewise polynomial interpolation in terms of accuracy. To illustrate this rigorously, we have the following lemma:

\begin{lemma}[\cite{hall1976optimal}]
Let $S(x) \in C^{2m-2}$ be the $(2m-1)$-th degree spline of $f$, as described by Eq. \eqref{eqn:polynomial_spline}, if $x_0, x_1, ...,x_{n}$ is a uniform partition with $x_{i+1}-x_i = h$, then
\begin{equation}
    ||f - S||_{\infty} \leq \frac{E_{2m}}{2^{2m}(2m)!}||f^{(2m)}||_\infty h^{2m},
\end{equation}
where $||\cdot||_{\infty}$ is the $l_\infty$ norm and $E_{2m} \sim (-1)^m \sqrt{\frac{m}{\pi}}(\frac{4m}{\pi e})^{2m}$ denotes the $2m$-th Euler number. The approximation error is bounded by $\mathcal{O}(h^{2m})$. For example,
\begin{equation}
    \begin{cases}
            ||f-S||_{\infty} \leq 0.25 \cdot h^{2},  \ \ \ \ \ \ \ \ \ \ \ \ \ \ \ \ m=1 \\
            ||f-S||_{\infty} \leq 0.013 \cdot h^{4},   \ \ \ \ \ \ \ \ \ \ \ \ \ \  m=2 \\
            ||f-S||_{\infty} \leq 0.00000137 \cdot h^{12}, \ \ \ \ \ \ m=6. \\
    \end{cases}
\end{equation}
\end{lemma}

This theorem also suggests that, to achieve the same accuracy, high-degree splines require fewer amounts of data than linear splines do. To reveal the expressivity of quadratic networks, we have the following proposition to show that a quadratic network can accurately express any univariate polynomial spline but a conventional network cannot. The method used for the proof is to re-express $S(x)$ into a summation of several continuous functions and use quadratic network modules to express these functions one by one. Finally, we aggregate these modules together, as shown by Figure \ref{Fig_proof_construction}.

\begin{figure}[htb]
\center{\includegraphics[width=\linewidth] {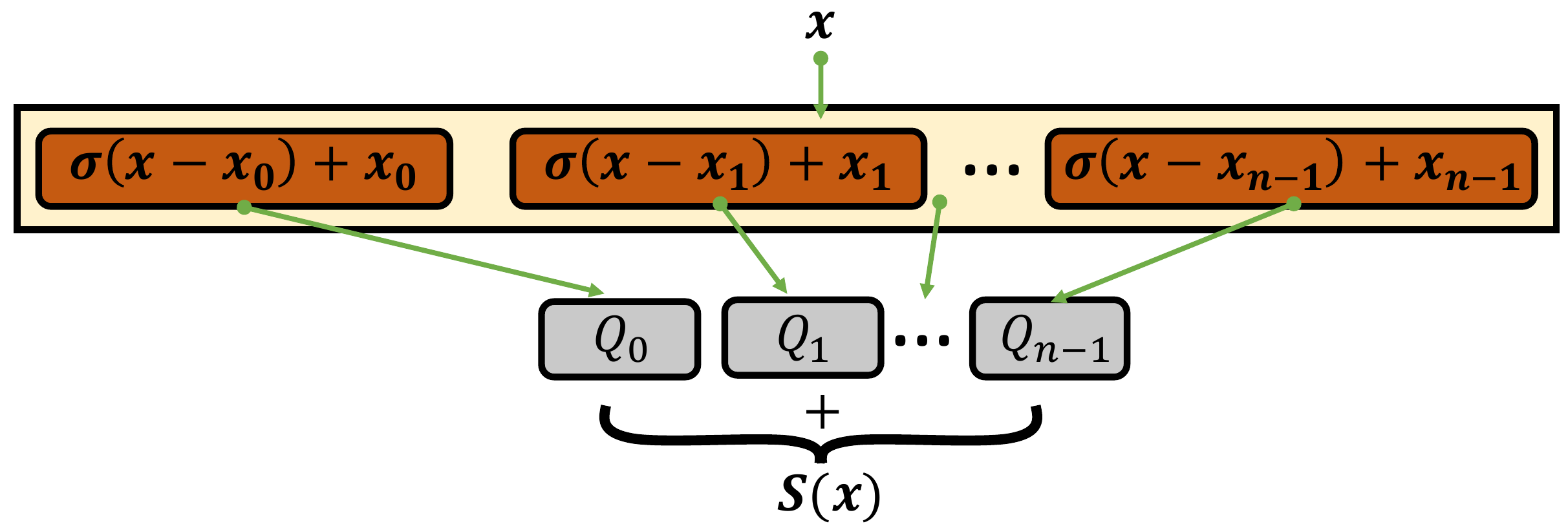}}
\caption{Constructing polynomial spline by a quadratic network.}
\label{Fig_proof_construction}
\end{figure}

\begin{prop}[Universal Spline Approximation]
Suppose that high-order terms of the polynomial spline do not degenerate, given a univariate polynomial spline $S(x)$ expressed in Eq. \eqref{eqn:polynomial_spline},
there exists a quadratic network with ReLU activation $Q(x)$ satisfying $Q(x)=S(x)$, while there exists no conventional networks with ReLU activation $L(x)$ such that $L(x)= S(x)$. 
\label{prop:spline_appro}
\end{prop}

\begin{proof}
Because a function defined by a quadratic network is a continuous function, we need to re-express $S(x)$ into a formulation on continuous functions to allow the construction. Mathematically, we re-write $S(x)$ given in Eq. \eqref{eqn:polynomial_spline} as
\begin{equation}
S(x) = \sum_{i=0}^{n-1} R_{i}(x),  \qquad x \in [x_0, x_n],
\label{eqn:re_polynomial_spline}        
\end{equation}
where for any $i = 0,\cdots, n-1$, 
\begin{equation}
       R_i(x) = \begin{cases}
        0   , \ \ \ \ \ \ \ \ \ \ \ \ \ \ \ \ \ \ \ \ \ \ \ \ \ x < x_i \\
        s_i (x)-s_{i-1}(x)    , \ \ \ \ \ \ \ \  x \geq x_i \\
        \end{cases}.
\label{eqn:re_polynomial}        
\end{equation}
For notation consistency, $s_{-1}(x)= 0$. It is straightforward to verify that Eq. \eqref{eqn:re_polynomial_spline} is equivalent to Eq. \eqref{eqn:polynomial_spline}. For any $x \in [x_k,x_{k+1})$,
\begin{equation}
\begin{aligned}
\sum_{i=0}^{n-1} R_{i}(x)
= &  \sum_{i=0}^{k} R_{i}(x)  =  \sum_{i=0}^{k} \big( s_i (x)-s_{i-1}(x)\big)  \\
= & s_k(x) = S(x).
\end{aligned}
\label{evaluation}
\end{equation}

$R_i(x)$ has the following favorable characteristics: 1) It is a truncated function that has zero function value when $x<x_i$; 2) and due to $s_i (x_i)-s_{i-1}(x_i) = 0$, $R_i(x)$ is also a continuous function. Thus, $R_i(x)$ can be succinctly expressed as
\begin{equation}
       R_i(x) = s_i (\sigma(x-x_i)+x_i)-s_{i-1}(\sigma(x-x_i)+x_i), 
\label{eqn:R_i}        
\end{equation}
where $\sigma(x-x_i)+x_i$ maps $x \in \{x|x\leq x_i\}$ into $x_i$ that has the function value of zero.

Because a quadratic network can represent any univariate polynomial \cite{fan2020universal}, we let ${Q_i}(x)=s_i (x)-s_{i-1}(x)$; then, 
\begin{equation}
       R_i(x) = Q_i(\sigma(x-x_i)+x_i). 
\label{eqn:Q_i}        
\end{equation}
Substituting Eq. \eqref{eqn:Q_i} into Eq.  \eqref{eqn:re_polynomial_spline}, we derive our construction:
\begin{equation}
Q(x)=s_0(x_0)+\sum_{i=0}^{n-1} Q_{i}(\sigma(x-x_i)+x_i)=S(x), 
\label{eqn:quadratic_polynomial_spline}        
\end{equation}
which concludes the proof of the first part of this proposition. For the second part, because a conventional network with ReLU activation is a piecewise linear function, $L(x)$ cannot perfectly represent $S(x)$ when high-order terms are non-zero.

\end{proof}

\textbf{Remark 1.} Although the proof of Proposition \ref{prop:spline_appro} is constructive and demands no complicated techniques, Proposition  \ref{prop:spline_appro} informs us an important message: A polynomial spline is a solution in the hypothesis space of quadratic networks, yet it will not appear in the hypothesis space of conventional networks. This implies that the expressivity of quadratic networks is superior to that of conventional networks, since a piecewise polynomial spline is certainly a better fitting tool than a piecewise linear spline. In a high-dimensional space, the polynomial spline is also superior to the linear spline \cite{wang2013multivariate}. Similarly, we can also use the quadratic network to construct the polynomial spline with no error but cannot use the conventional network to achieve so.

\subsection{Dimension of Algebraic Variety}

To our best knowledge, there exist at least two ways to realize so-called polynomial networks. The first is to utilize a polynomial activation function, while the second one is to take a polynomial function for the aggregation, such as our quadratic neurons. Despite the same name, two realizations are dramatically different. To put the superior expressivity of quadratic networks in perspective, we employ \textit{the dimension of algebraic variety}, which was proposed to gauge the expressive power of polynomial networks \cite{kileel2019expressive}, to compare the two realizations. We find that the dimension of algebraic variety of our quadratic network is significantly higher than that of its competitor, suggesting that our quadratic network can represent a richer class of functions than the network using quadratic activation.

\textbf{Two realizations.} Assume that a network architecture consists of $H$ layers with their widths specified by a vector $\bm{d} = (d_0, d_1, d_2, \cdots, d_H)$ and $\x \in \mathbb{R}^{d_0}$. A network with a quadratic activation is a function of the form
\begin{equation}
    p_1(\x)= l_H \circ \sigma_1 \circ l_{H-1} \circ \sigma_1 \circ l_{H-2} \cdots \circ \sigma_1 \circ l_1 (\x),
\label{eqn:quadratic_1}    
\end{equation}
where $p_1(\x)\in \mathbb{R}^{d_H}$, $l^h (\x) = \bm{W}^h \x + \bm{b}^h$, and $\sigma_1(z) = z^2$. In contrast, our quadratic network is of the following form: 
\begin{equation}
    p_2(\x)= q_H \circ \sigma_2 \circ q_{H-1} \circ \sigma_2 \circ q_{H-2} \cdots \circ \sigma_2 \circ q_1 (\x),
\label{eqn:quadratic_2}
\end{equation}
where $p_2(\x)\in \mathbb{R}^{d_H}$, $q_h (\x) = (\bm{W}^{h,r} \x + \bm{b}^{h,r})\odot(\bm{W}^{h,g} \x + \bm{b}_{h,g})+\bm{W}^{h,b} (\x\odot\x)+\bm{c}^h$, and $\sigma_2(z) = z$ (To simplify our analysis, we use linear activation for our quadratic network as $\sigma_2(x) = \text{ReLU}(x) - \text{ReLU}(-x)$). Given an architecture $\bm{d}$, the polynomial network with respect to the weights and biases defines a functional space, and we denote the functional spaces of two realizations as $\mathcal{F}_{\bm{d},1}$ and $\mathcal{F}_{\bm{d},2}$, respectively.

\textbf{Dimension of algebraic variety.} In \cite{kileel2019expressive}, the Zariski closure $\mathcal{V}_{\bm{d}}=\overline{\mathcal{F}_{\bm{d}}}$ of $\mathcal{F}_{\bm{d}}$ is considered, where $\mathcal{V}_{\bm{d}}$ is an algebraic variety, and the dimension of $\mathcal{V}_{\bm{d}}$ ($\dim \mathcal{V}_{\bm{d}}$) is a measure to the expressivity of the pertaining network. Although $\mathcal{V}_{\bm{d}}$ is larger than $\mathcal{F}_{\bm{d}}$, their dimensions agree with each other. Moreover, $\mathcal{V}_{\bm{d}}$ is amendable to the powerful analysis tools from algebraic geometry. In the following, based on the results in \cite{kileel2019expressive}, we provide an estimation for the upper bound of $\dim \mathcal{V}_{\bm{d},2}$.

\begin{lemma}
Given an architecture $\bm{d}= (d_0, d_1, d_2, \cdots, d_H)$, the following holds:
\begin{equation}
\begin{small}
\begin{aligned}
& \dim \mathcal{V}_{\bm{d},2} \leq \\
& \min \Big( \sum_{h=1}^H (3d_{h-1}d_h+3d_{h-1})-\sum_{h=1}^{H-1} d_h, d_H{{d_0+2^{H-1}-1}\choose{2^{H-1}}} \Big).
\end{aligned}
\label{eqn:bound}
\end{small}
\end{equation}
\end{lemma}

\begin{proof}
For all diagonal matrices $D_i \in \mathbb{R}^{d_i \times d_i}$ and permutation matrices $P_i \in \mathbb{Z}^{d_i \times d_i}$, the function described in \eqref{eqn:quadratic_2} returns the same output under the following replacements:
\begin{equation}
    \begin{cases}
    & \bm{\Theta}_{1} \leftarrow P_1 D_1 \bm{\Theta}_{1}, \\
    & \bm{\Theta}_{2} \leftarrow P_2 D_2 \bm{\Theta}_{2}D_1^{-1}P_1^T, \\
    & \bm{\Theta}_{3} \leftarrow P_3 D_3 \bm{\Theta}_{3}D_2^{-1}P_2^T, \\
    & \ \ \ \ \ \ \ \ \ \vdots \\
    & \bm{\Theta}_{H} \leftarrow  \bm{\Theta}_{H}D_{H-1}^{-1}P_{H-1}^T,
    \end{cases}
\end{equation}
where $\bm{\Theta}_{h}$ represents any element in $\{\bm{W}^{h,r}, \bm{W}^{h,g}, \bm{W}^{h,b}, \bm{b}^{h,r}, \bm{b}^{h,g}, \bm{c}^{h}\}$. 
As a result, the dimension of a generic fiber of $p_2(\x)$ is at least $\sum_{i=1}^{H-1} d_i$.

According to \cite{eisenbud2013commutative}, the dimension of $\mathcal{V}_{\bm{d},2}$, $\dim \mathcal{V}_{\bm{d},2}$, is equal to the dimension of the domain of $p_2(\x)$ minus the dimension of the generic fiber of $p_2(\x)$, which means
\begin{equation}
 \dim \mathcal{V}_{\bm{d},2} \leq \sum_{h=1}^H (3d_{h-1}d_h+3d_{h-1})-\sum_{h=1}^{H-1} d_h   
\end{equation}
In addition, $\dim \mathcal{V}_{\bm{d},2}$ is at most the number of terms of $p_2(\x)$, which means
\begin{equation}
    \dim \mathcal{V}_{\bm{d},2} \leq d_H{{d_0+2^{H-1}-1}\choose{2^{H-1}}}.
\end{equation}
Combining the above two formulas, we conclude this proof.
\end{proof}

For the same architecture, the upper bound provided by the network with quadratic activation \cite{kileel2019expressive} is
\begin{equation}
 \min \Big( \sum_{h=1}^H (d_{h-1}d_h+d_{h-1})-\sum_{h=1}^{H-1} d_h, d_H{{d_0+2^{H-1}-1}\choose{2^{H-1}}} \Big).
\end{equation}
This bound is lower than what we derived in Eq. \eqref{eqn:bound}. In addition to the upper bound comparison, we have the following proposition to directly compare $\dim \mathcal{V}_{\bm{d},2}$ and $\dim \mathcal{V}_{\bm{d},1}$. 

\begin{prop}
Given the same architecture $\bm{d}= (d_0, d_1, d_2, \cdots, d_H)$, we have
\begin{equation}
\dim \mathcal{V}_{\bm{d},2} > \dim \mathcal{V}_{\bm{d},1}.
\end{equation}
\end{prop}

\begin{proof}
It can be shown that by the following substitutions:
\begin{equation}
    \begin{cases}
    & \bm{W}^{h,r}= \bm{W}^{h,g} \leftarrow \bm{W}^{h}, \bm{W}^{h,b} \leftarrow \bm{0} \\
    & \bm{b}^{h,r} = \bm{b}^{h,g} \leftarrow \bm{b}^{h}, \bm{c} \leftarrow \bm{0},
    \end{cases}
\end{equation}
$p_2(\x)$ turns into $p_1(\x)$, which means $\mathcal{F}_{\bm{d},1} \subset \mathcal{F}_{\bm{d},2}$.

However, getting $p_2(\x)$ from $p_1(\x)$ is difficult because we need to construct interaction terms $\x_k \x_l$ from $p_1(\x)$. Generally, representing a single quadratic neuron with neurons using quadratic activation requires a good number of neurons:
\begin{equation}
\begin{aligned}
     &  (\sum_{i=1}^{n} w_{i}^r x_i +b^r)(\sum_{i=1}^{n} w_{i}^g x_i +b^g) + \sum_{i=1}^{n} w_{i}^b x_{i}^2+c  \\
   = &  \sum_{k\neq l} (w_{k}^r w_{l}^g+w_{l}^r w_{k}^g) x_k x_l +  \sum_k (w_{k}^r b^g+w_k^g b^r) x_k \\
    & + \sum_{k} (w_{k}^r w_{k}^g+w_k^b) x_k^2 + c \\
   = &  \sum_{k\neq l} (A_{k}x_k+B_{l}x_l)^2 + \sum_k (C_k x_k - D_k)^2 + \sum_k E_k x_k^2,    
\end{aligned}
\end{equation}
where $A_{k}, B_{l}, C_{k}, D_{k}, E_{k}$ are coefficients.

From another meaningful angle, $p_1(x)$ can only get degree $2^k$ polynomials, while $p_2(x)$ can be an arbitrary degree polynomial because there is a product operation in the quadratic neuron. Therefore, $p_1(x)$ can never represent $p_2(x)$. As a result, given the same network structure, $\mathcal{V}_{\bm{d},1} \subset \mathcal{V}_{\bm{d},2}$, thereby we can conclude that $\dim \mathcal{V}_{\bm{d},2} > \dim \mathcal{V}_{\bm{d},1}$.

\end{proof}

\textbf{Remark 2.} The functions defined by the polynomial networks are functional varieties, and their Zariski closures are algebraic varieties. Using the dimension of the algebraic variety to measure its capacity, yet simple and straightforward, we show that our quadratic network has higher expressivity than the conventional network with quadratic activation. The picture is even clearer when the network architecture is shallow and with infinite width. According to  \cite{siegel2020approximation}, a network with quadratic activation is never dense in the set of continuous functions. However, a quadratic network using ReLU activation equipped with such an architecture is preferably dense.

\section{TRAINABILITY OF QUADRATIC NETWORKS}

\begin{figure}[htb]
\center{\includegraphics[width=\linewidth] {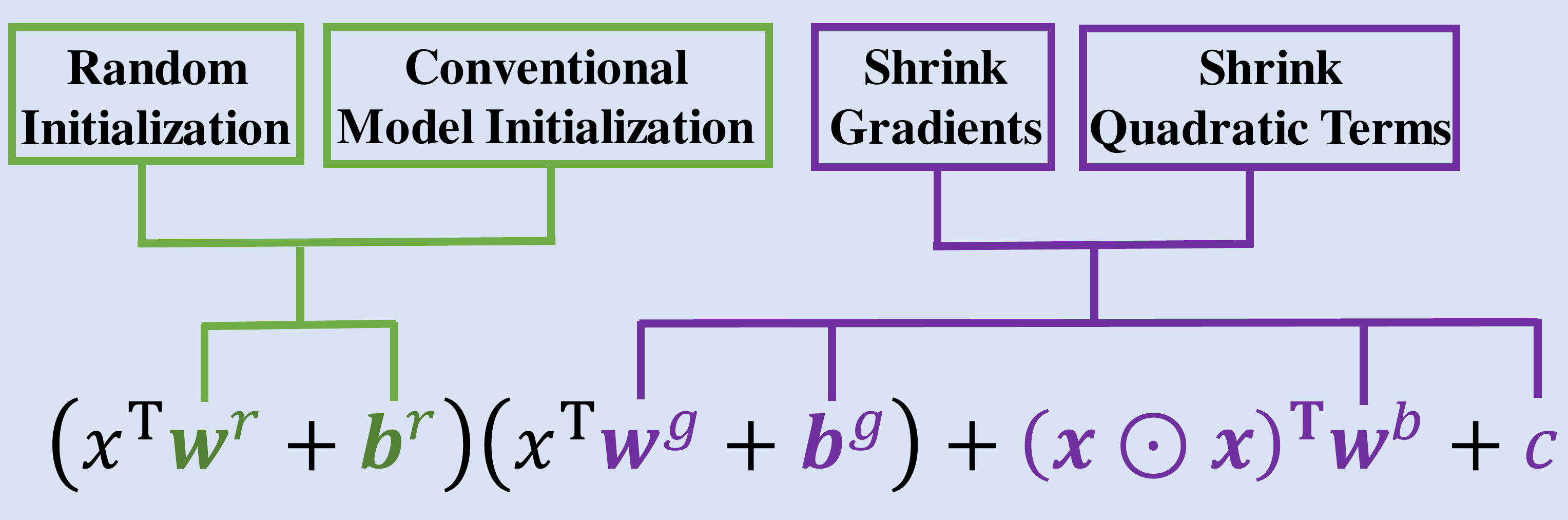}}
\caption{Illustration of the proposed training strategy.}
\label{Fig_training_strategy}
\vspace{-0.3cm}
\end{figure}

\begin{table*}
\centering

\caption{Proposed training strategy. ReLinear$^{sg}$ uses shrinking gradients, while ReLinear$^{sw}$ works with shrinking weights.}
  \begin{tabular}{c|cc|c|c}
    \hline
     &
      \multicolumn{2}{c|}{Initialization} &
      \multicolumn{1}{c|}{Learning Rate} & Updating Equation \\
      
    & $\textbf{w}^g, b^g$ &  $\textbf{w}^b, c$  & $\gamma_g$, $\gamma_b$ \\
    \hline
    ReLinear$^{sg}$ &   $\bm{0}, 1$ & $\bm{0}, 0$  & \makecell{$\gamma_g < \gamma_r$\\  $\gamma_b < \gamma_r$} & \makecell{$\textbf{w}^r = \textbf{w}^r - \gamma_r\frac{\partial{\mathcal{L}}}{\partial{{\textbf{w}^r}}}$ \\
    $\textbf{w}^g = \textbf{w}^g - \gamma_g\frac{\partial{\mathcal{L}}}{\partial{{\textbf{w}^g}}}$ \\ 
    $\textbf{w}^b = \textbf{w}^b - \gamma_b\frac{\partial{\mathcal{L}}}{\partial{{\textbf{w}^b}}}$}\\
    \hline
    ReLinear$^{sw}$-$l_1$  & $\bm{0},1$ & $\bm{0}, 0$  & $\gamma_g = \gamma_b = \gamma_r$ & \makecell{$\textbf{w}^r = \textbf{w}^r - \gamma_r\frac{\partial{\mathcal{L}}}{\partial{{\textbf{w}^r}}}$ \\ 
    $\textbf{w}^g = \textbf{w}^g - \alpha_g\cdot sgn(\textbf{w}^g) - \gamma_r\frac{\partial{\mathcal{L}}}{\partial{{\textbf{w}^g}}}$  \\
    $\textbf{w}^b = \textbf{w}^b - \alpha_b\cdot sgn(\textbf{w}^b) - \gamma_r\frac{\partial{\mathcal{L}}}{\partial{{\textbf{w}^b}}}$}  \\
    \hline
    ReLinear$^{sw}$-$l_2$ & $\bm{0},1$ & $\bm{0}, 0$  & $\gamma_g = \gamma_b = \gamma_r$ & \makecell{$\textbf{w}^r = \textbf{w}^r - \gamma_r\frac{\partial{\mathcal{L}}}{\partial{{\textbf{w}^r}}}$ \\ 
    $\textbf{w}_g = \textbf{w}^g (1-\beta_g) - \gamma_r\frac{\partial{\mathcal{L}}}{\partial{{\textbf{w}^g}}}$  \\
    $\textbf{w}^b = \textbf{w}^b (1-\beta_b) - \gamma_r\frac{\partial{\mathcal{L}}}{\partial{{\textbf{w}^b}}}$}  \\
    \hline
  \end{tabular}
     \vspace{1ex}
 
  {Updating $b^r, b^g, c$ can be similarly done in reference to the  equations for updating $\textbf{w}^r, \textbf{w}^g$ and $\textbf{w}^b$ respectively. \par}

  \label{training_strategy}
  \vspace{-0.5cm}
\end{table*}

Despite the superior expressivity, the training of a quadratic network may collapse, which prevents a quadratic network from achieving its full potential. When randomly initializing parameters of a quadratic network, the training is typically unstable: sometimes the model yields an exploding magnitude of the output; and in some other cases the training curve oscillates. Likely, this is because a quadratic term is nonlinear, and the composition of quadratic operations layer by layer produces a function of an exponentially high degree, causing the instability of the training process. As such, although a quadratic operation is more powerful and promises superior performance, it is necessary to balance model scalability and training stability.  

\subsection{The Training Strategy: ReLinear} 

To control the quadratic terms, we propose the \textit{ReLinear} (referenced linear initialization), which encourages the model to learn suitable quadratic terms gradually and adaptively in reference to the corresponding linear terms. The ReLinear method has the following two steps. First, as shown in Figure \ref{Fig_example_oscillation}, the quadratic neurons are sensitive to a slight change in initialization. Therefore,
the quadratic weights in each neuron are set to $\textbf{w}^g = 0, b^g = 1$ and $\textbf{w}^b = 0, c = 0$. A good initialization is important for the network \cite{kumar2017weight}. Such an initialization degenerates a quadratic neuron into a conventional neuron. Second, quadratic terms are regularized in the training process. Intuitively, two ways of regularization: shrinking the gradients of quadratic weights (ReLinear$^{sg}$); and shrinking quadratic weights (ReLinear$^{sw}$). Let $\gamma_r$, $\gamma_g$, and $\gamma_b$ be the learning rates for updating $\textbf{w}^r, b^r$, $\textbf{w}^g, b^g$ and $\textbf{w}^b, c$, respectively,  $\alpha_g, \alpha_b$ be the weight factors of $\textbf{w}_g$ and $\textbf{w}_b$, and $\beta_g, \beta_b$ be the weight decay rates for $\textbf{w}_g$ and $\textbf{w}_b$, respectively. In Table \ref{training_strategy}, we summarize the key points of the proposed training strategy. 

Specifically, for ReLinear$^{sg}$ in Table \ref{training_strategy}, we set different learning rates for $\textbf{w}^r, b^r$ and $\textbf{w}^g, b^g, \textbf{w}^b, c$, where the learning rate for the former keeps intact as the conventional network, while the learning rate for the latter adjusts quadratic nonlinearity. By setting the learning rate of $\textbf{w}^g, b^g, \textbf{w}^b, c$ to a smaller value, we can prevent magnitude explosion while utilizing the quadratic nonlinearity. For ReLinear$^{sw}$ in Table \ref{training_strategy}, a straightforward way is to use $l_1$ or $l_2$ norms for $\textbf{w}^g, b^g, \textbf{w}^b, c$, and shrink the quadratic weights at each iteration directly.

Shrinking the gradients of quadratic terms (ReLinear$^{sg}$) is better than shrinking quadratic terms (ReLinear$^{sw}$). Shrinking quadratic terms is to adjust quadratic terms along the directions of those parameters, which may not decrease the loss due to their deviation from the directions of the gradients. In contrast, shrinking the gradients respects the directions of the gradients, which works better to minimize the loss. 

Furthermore, regarding the parameters $\textbf{w}^r, b^r$, we can use either random initialization or weight transfer to train a conventional network sharing the same structure of a quadratic network and then transfer the learned parameters into the quadratic network. Thus, in weight transfer, $\textbf{w}^r, b^r$ in each quadratic neuron are initialized by the parameters of the corresponding conventional neuron. In contrast to the random initialization, the weight transfer has an extra computational cost to train the conventional model. If the conventional model of the same structure needs to be trained, we estimate that the total cost will increase by around $20\%$ because the number of multiplications of a conventional neuron is $20\%$ of that of a quadratic neuron. 

\textbf{Remark 3.} The pre-trained models are widely used in many computationally intensive or data-insufficient scenarios \cite{chen2021pre}. For example, in transfer learning, the representation and knowledge from a pre-trained model for a task can facilitate a model for another task. To train a quadratic network, we may also use a pre-trained conventional model for weight transfer. However, doing so is not a practice of transfer learning, as the pre-trained model is from the same task.

In our numerical experiments, consistent with Figure \ref{fig:guaranteed}, the quadratic network trained via ReLinear always outperforms the conventional network of the same structure. Because when $\gamma_{g}=\gamma_{b}=0$ or $\alpha_{g}=\alpha_{b}=0$ or $\beta_{g}=\beta_{b}=0$, the quadratic network will be a conventional network, therefore at least delivering the same performance as the conventional network. As $(\gamma_{g}, \gamma_{b})$ or $(\alpha_{g}, \alpha_{b})$, or $(\beta_{g},\beta_{b})$ gradually increases, the quadratic terms are preferably evolving to extract features and refine the workflow, making the model generally better than the corresponding conventional model. Of course, $(\gamma_{g}, \gamma_{b})$ or $(\alpha_{g}, \alpha_{b})$, or $(\beta_{g},\beta_{b})$ should not be too large or small; otherwise, the quadratic terms would be insignificant or too aggressive. Moreover, tuning these hyperparameters is easy and will not undermine user-friendliness.

\subsection{Mechanism Analysis} 
In this part, we theoretically shed light on why the proposed ReLinear can avoid the magnitude explosion during the training of quadratic networks. We also analyze the convergence behavior of the proposed ReLinear and its corresponding convergence rate. 

\textbf{Stabilizing the training.} For conciseness and convenience, notations in Eqs. \eqref{eqn:quadratic_1} and \eqref{eqn:quadratic_2} are inherited. We formulate a fully-connected conventional network as 
\begin{equation}
    f_1(\x)= l_H \circ \sigma_1 \circ l_{H-1} \circ \sigma_1 \circ l_{H-2} \cdots \circ \sigma_1 \circ l_1 (\x),
\end{equation}
where $l^h (\x) = \bm{W}^h \x + \bm{b}^h$, $\bm{W}^h \in \mathbb{R}^{d_h \times d_{h-1}}, \bm{b}^h \in \mathbb{R}^{d_h}$, and $\sigma_1(z) = \max\{z,0\}$.
A fully-connected quadratic network is formulated as 
\begin{equation}
    f_2(\x)= q_H \circ \sigma_2 \circ q_{H-1} \circ \sigma_2 \circ q_{H-2} \cdots \circ \sigma_2 \circ q_1 (\x),
\end{equation}
where $q_h (\x) = (\bm{W}^{h,r} \x + \bm{b}^{h,r})\odot(\bm{W}^{h,g} \x + \bm{b}^{h,g})+\bm{W}^{h,b} (\x\odot\x)+\bm{c}^h$, $\bm{W}^{h,r}, \bm{W}^{h,g}, \bm{W}^{h,b} \in \mathbb{R}^{d_h \times d_{h-1}}, \bm{b}^{h,r}, \bm{b}^{h,g}, \bm{c}^h \in \mathbb{R}^{d_h}, h=1,2,\cdots,H$, $d_H=1$, and $\sigma_2(z) = \sigma_1(z)$.

In our experiments, we find that the training loss of a normally trained quadratic network can be very large in the early stage because the output of the network is large. We think this is due to the nonlinear amplification of quadratic terms to the input. The reason why the ReLinear can stabilize the training of quadratic networks is its ability to suppress the quadratic terms properly. To dissect the hidden mechanism, we first conduct Taylor expansion around $0$ for conventional and quadratic networks by keeping the linear term (A ReLU network is not differentiable, but w.l.o.g., we can assume the ReLU is approximated by a smooth function):
\begin{equation}
 f(\Delta \x)  \approx f(0)  + \frac{\partial f}{\partial \x}\Big|_{\x = 0} \Delta \x.
\end{equation}

Naturally, $\Vert \frac{\partial f}{\partial \x}\Vert_{\x = 0}$ is used to measure the amplification effect. The derivatives of $ f_1(\x)$ and $f_2(\x)$ are computed as follows:
\begin{equation}
\begin{aligned}
  &  \Vert \frac{\partial f_1}{\partial \x}\Vert_{\x = 0}\\
  = & \Vert \bm{W}^H \odot \sigma'(\u^{H-1})^\top \bm{W}^{H-1} \odot \sigma'(\u^{H-2})^\top\cdots \bm{W}^1 \Vert   \\
  \leq  & \Vert \bm{W}^{H} \Vert \Vert \bm{W}^{H-1} \Vert \cdots \Vert \bm{W}^1 \Vert
\end{aligned}
\label{eqn:linear_derivative_reduced}
\end{equation}
and 
\begin{equation}
\begin{aligned}
         &  \Vert \frac{\partial f_2}{\partial \x}\Vert_{\x = 0}  \\ =  & \Vert (\bm{W}^{H,r}\odot (\bm{W}^{H,g} \u^{H-1} + \bm{b}^{H,g})+ \bm{W}^{H,g}\odot  (\bm{W}^{H,r} \u^{H-1}  \\
        & + \bm{b}^{H,r})  + 2\bm{W}^{H,b}\odot \u^{H-1})  
    \odot  \sigma'(\u^{H-1})^\top \cdots  (\bm{W}^{1,r}\odot \\ &  (\bm{W}^{1,g} \x  +  \bm{b}^{1,g})+ \bm{W}^{1,g}\odot (\bm{W}^{1,r} \x + \bm{b}^{1,r})+2\bm{W}^{1,b}\odot\x) \Vert \\
    \leq & \Vert (\bm{W}^{H,r}\odot (\bm{W}^{H,g} \u^{H-1} + \bm{b}^{H,g})+ \bm{W}^{H,g}\odot  (\bm{W}^{H,r} \u^{H-1}  \\
        & + \bm{b}^{H,r})  + 2\bm{W}^{H,b}\odot \u^{H-1})  \Vert 
     \cdots \Vert (\bm{W}^{1,r}\odot  (\bm{W}^{1,g} \x  +  \bm{b}^{1,g}) \\ & + \bm{W}^{1,g}\odot (\bm{W}^{1,r} \x + \bm{b}^{1,r})+2\bm{W}^{1,b}\odot\x) \Vert, \\
\end{aligned}
\label{eqn:quadratic_derivative}
\end{equation}
where $\u^h$ is the output the $h$-th layer of networks, and the Hadamard product\footnote{https://rdrr.io/cran/FastCUB/man/Hadprod.html} between the given matrix and vector corresponds to multiply every column of the matrix by the corresponding element of the vector. Because $\x$ is around 0, we assume that $\u^{h}$ is infinitesimal, then Eq.\eqref{eqn:quadratic_derivative} becomes
\begin{equation}
\begin{aligned}
         \Vert \frac{\partial f_2}{\partial \x}\Vert_{\x = 0}  \leq & \Vert (\bm{W}^{H,r} \odot \bm{b}^{H,g}+ \bm{W}^{H,g} \odot \bm{b}^{H,r})\Vert \cdots  \\
    & \Vert (\bm{W}^{1,r} \odot \bm{b}^{1,g}+ \bm{W}^{1,g} \odot \bm{b}^{1,r}) \Vert .
\end{aligned}
\label{eqn:quadratic_derivative_reduced}
\end{equation}
Per the initialization of the ReLinear, when $\bm{W}^{h,g}=\bm{0}, \bm{b}^{h,r}=\bm{1}$, $\Vert \frac{\partial f_2}{\partial \x} \Vert = \Vert \frac{\partial f_1}{\partial \x} \Vert$, which means that at the beginning, the amplification effects of conventional and quadratic networks are equal.

Furthermore, we derive the Frobenius norms of all concerning matrices and vectors in Eqs. \eqref{eqn:linear_derivative_reduced} and \eqref{eqn:quadratic_derivative_reduced}. Suppose that the parameters of two networks are i.i.d. Gaussian distributed with unit variance, the typical magnitude of the Euclidean norm of $\bm{W} \in \mathbb{R}^{m\times n}$ is $\mathcal{O}(\sqrt{mn})$, and the typical magnitude of the Euclidean norm of $\bm{W}\odot\bm{b}$, where $\bm{W} \in \mathbb{R}^{m\times n}, \bm{b}\in \mathbb{R}^{n\times 1}$ is also $\mathcal{O}(\sqrt{mn})$. Thus, we estimate the magnitudes of $\Vert \frac{\partial f_1}{\partial \x}\Vert_{\x = 0}$ and $\Vert \frac{\partial f_2}{\partial \x}\Vert_{\x = 0}$ as
\begin{equation}
\begin{cases}
     &   \Vert \frac{\partial f_1}{\partial \x}\Vert_{\x = 0} \approx
    \prod_{h=1}^H \sqrt{d_h d_{h-1}} \\
    & \Vert \frac{\partial f_2}{\partial \x}\Vert_{\x = 0} \approx
    \prod_{h=1}^H 2\sqrt{d_h d_{h-1}} = 2^H \prod_{h=1}^H \sqrt{d_h d_{h-1}}.
\end{cases}
\label{eqn:amplifying}
\end{equation}
Based on Eq. \eqref{eqn:amplifying}, the amplification effect of a quadratic network is exponentially higher than that of a conventional network, which accounts for the magnitude explosion of a quadratic network during training. The ReLinear regularizes quadratic terms by a prescribed initialization and weight/gradient shrinkage, which favorably avoids the exponential growth of the amplification effect at least at the early training stage. This is why the ReLinear can work.

\textbf{Convergence.} The ReLinear strategy is essentially the stochastic gradient descent (SGD) with a specific initialization and adaptive learning rates for different parameters. The convergence behavior is mainly determined by adaptive learning rates and initialization. In Appendix, aided by the proofs in \cite{li2019convergence}, we show under what conditions the proposed ReLinear can converge and the associated convergence rate.

\section{EXPERIMENTS}
In this section, we first conduct analysis experiments (Runge function approximation and image recognition) to show the effectiveness of the ReLinear strategy in controlling quadratic terms and to analyze the performance of different schemes for the proposed ReLinear method. Then, encouraged by our theoretical analyses, we compare quadratic networks with conventional networks and the networks with quadratic activation to show that a quadratic network is a competitive model. 

\subsection{Analysis Experiments}

\subsubsection{Runge Function Approximation}

As mentioned earlier, a polynomial spline is used to replace a complete polynomial to overcome the Runge phenomenon. Since a quadratic network with ReLU activation is a polynomial spline, we implement a fully connected quadratic network to approximate the Runge function to verify the feasibility of our proposed training strategy in suppressing quadratic terms. This experiment is to approximate a univariate function, which enables us to conveniently compute the degree of the output function produced by a quadratic network and monitor its change. 

In total, $33$ points are sampled from $[-5,5]$ with an equal distance. The width of all layers is $8$. The depth is $5$ such that the degree of the output function is $2^5=32$, meeting the minimum requirement to fit $33$ instances. We compare the proposed strategy (ReLinear$^{sw}$-$l_1$, ReLinear$^{sw}$-$l_2$, and ReLinear$^{sg}$) with the regular training. In ReLinear$^{sw}$-$l_1$, $\gamma_r = \gamma_g = \gamma_b = 0.0003$ and $\alpha_g = \alpha_b = 0.0001$. In ReLinear$^{sw}$($l_2$), $\gamma_r = \gamma_g = \gamma_b = 0.0003$ and $\beta_g = \beta_b = 0.0001$. In ReLinear$^{sg}$, the learning rates are set as $\gamma_r = 0.0003, \gamma_g = \gamma_b = 0.00015$. In contrast, we configure a learning rate of $\gamma_r = \gamma_g = \gamma_b = 0.0003$ for all parameters in regular training. The total iteration is $30,000$ to guarantee convergence.

\begin{figure}[htb]
\center{\includegraphics[width=\linewidth] {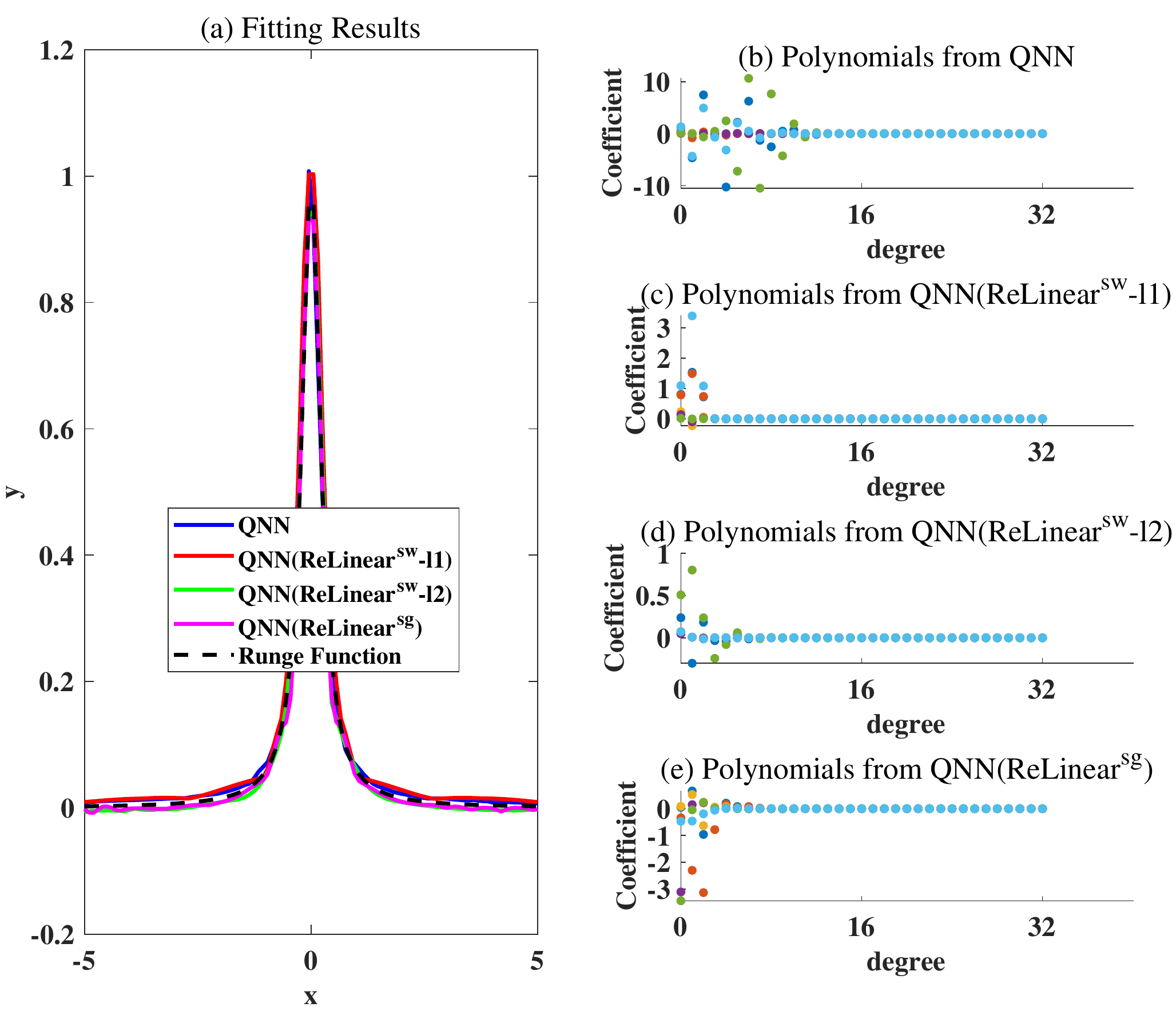}}
\caption{(a) Fitting by QNN via different training strategies. (b)-(e) Coefficients of randomly selected piecewise polynomials from QNN trained with regular training and the proposed strategies (ReLinear$^{sw}$-$l_1$, ReLinear$^{sw}$-$l_2$ and ReLinear$^{sg}$).}
\label{RungeFitting}
\end{figure}

The results are shown in Figure \ref{RungeFitting}. As a spline can avoid the Runge phenomenon, regardless of how the QNN with ReLU activation is trained, it can fit the Runge function desirably without oscillations at edges, as shown in Figure \ref{RungeFitting}(a). Furthermore, in Figure \ref{RungeFitting}(b)-(e), we examine coefficients of randomly selected polynomials contained by functions of QNNs at different pieces. It is observed from Figure \ref{RungeFitting}(b) that the polynomials associated with the regular training have unignorable high-order terms. This is counter-intuitive because a QNN partitions the interval into many pieces (24 pieces based on our computation). Since only a few samples lie in each piece, it suffices to use a low-degree polynomial in each piece. This might be due to that the space of low-degree polynomials is a measure-zero subspace in the high-degree polynomials. Thus, it is hard to obtain a low-degree polynomial fit straightforwardly.
Next, we observe that coefficients of high degrees are significantly suppressed in Figure \ref{RungeFitting}(c)-(e) than those in Figure \ref{RungeFitting}(b). At the same time, the magnitudes of coefficients of low degrees in Figure \ref{RungeFitting}(c)-(e) are put down. Such observations imply that all the proposed strategies can effectively control the quadratic terms as expected.  

\begin{table}[htb]
\centering

\caption{RMSE valuse of different training strategies.}
  \begin{tabular}{c|c|c|c}
    \hline
    Training &  ReLinear$^{sw}$($l_1$)  & ReLinear$^{sw}$($l_2$) & ReLinear$^{sg}$  \\
        \hline
    RMSE & 0.0656   & 0.0426 & 0.0205  \\

    \hline

  \end{tabular}
  \label{appro_performance}
\end{table}

Next, we quantitatively compare the approximation errors of different training strategies using the rooted mean squared error (RMSE). We evenly sample 100 instances from $[-5,5]$ as test samples, none of which appears in the training. Table \ref{appro_performance} shows RMSE values of three realizations for ReLinear. It can be seen that ReLinear$^{sg}$ achieves the lowest error, suggesting that ReLinear$^{sg}$ is better at balancing between suppressing quadratic terms and maintaining approximation precision.

\subsubsection{Image Classification}

Here, we focus on an image recognition task to further confirm the effectiveness of the proposed strategy. We build a quadratic ResNet (QResNet for short) by replacing conventional neurons in ResNet with quadratic neurons and keeping everything else intact. We train the QResNet on the CIFAR10 dataset. Our motivation is to gauge the characteristics and performance of different variants of the ReLinear method through experiments.

Following configurations of \cite{he2016deep}, we use batch training with a batch of 128. The optimization is stochastic gradient descent using Adam \cite{kingma2014adam}. $\gamma_r$ is set to $0.1$. The total number of epochs is $200$. In the $100$-th and $150$-th epoch, $\gamma_r, \gamma_g, \gamma_b$ decrease $1/10$ at the same time. Because training curves share the same trend as the testing curves, we only show the testing curves here for conciseness. 

\textit{Tuning ReLinear$^{sg}$.} Here, with QResNet20, we study the impact of $\gamma_g, \gamma_b$ on the effectiveness of ReLinear$^{sg}$. Without loss of generality, we set $\gamma_g=\gamma_b$. The lower $\gamma_g$ and $\gamma_b$ are, the more quadratic weights are constrained. We respectively set $\gamma_g, \gamma_b$ to $\{10^{-1}, 10^{-2}, 10^{-3}, 10^{-4}, 10^{-5}\}$ for a comprehensive analysis. The resultant accuracy curves are shown in Figure \ref{fig:compare_sg}. It can be seen that when $\gamma_g$ is large ($\gamma_g=\gamma_b=10^{-1}, 10^{-2}, 10^{-3}$), the training is quite unstable, the accuracy score jumps severely. However, as $\gamma_g$ and $\gamma_b$ go low, \textit{i.e.}, $\gamma_g=\gamma_b=10^{-4}, 10^{-5}$, the training curves become stabilized, mirroring that the high-order terms are well controlled. The best performance (error $7.78\%$) is achieved when $\gamma=10^{-4}$, consistent with the trend in Figure \ref{fig:guaranteed}.

\begin{figure}[htb]
\center{\includegraphics[width=\linewidth] {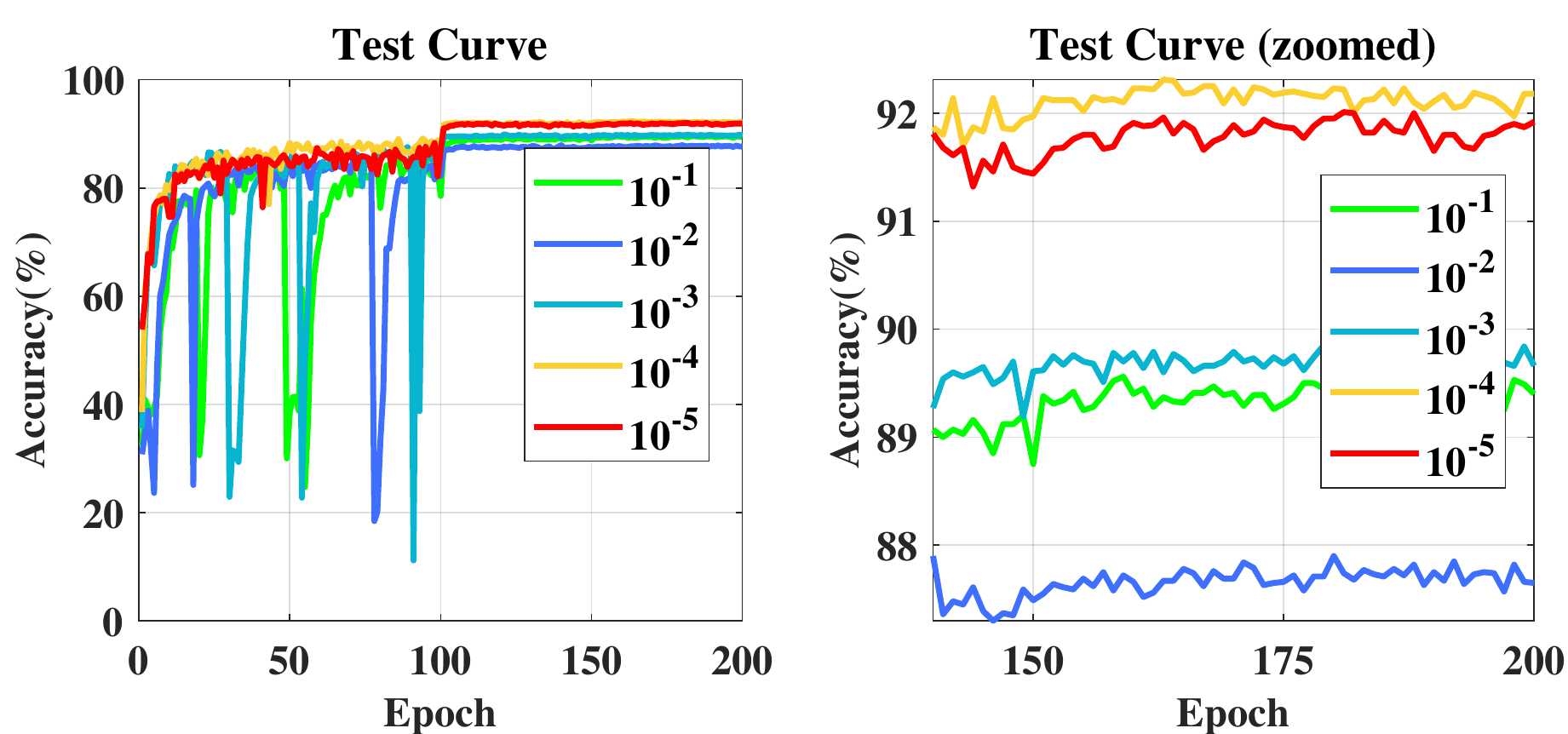}}
\caption{Accuracy curves obtained from different learning rates ($\gamma_g, \gamma_b$) for ReLinear$^{sg}$.}
\label{fig:compare_sg}
\end{figure}

\textit{Tuning ReLinear$^{sw}$.} Here, we investigate the impact of different parameters for ReLinear$^{sw}$-$l_1$ and ReLinear$^{sw}$-$l_2$, respectively. For both ReLinear$^{sw}$-$l_1$ and ReLinear$^{sw}$-$l_2$, $\gamma_r$, $\gamma_g$, and $\gamma_b$ are set to $0.1$, and decay $1/10$ at epochs $100$ and $150$ at the same time. Concerning parameters of ReLinear$^{sw}$-$l_1$, let $\alpha_g=\alpha_b$ which are respectively set to $\{10^{-1}, 10^{-2}, 10^{-3}, 10^{-4}, 10^{-5}\}$. For ReLinear$^{sw}$-$l_2$, we cast $\beta_g=\beta_b$ to $\{0.1, 0.5, 0.9, 0.99\}$, respectively. If the $\beta_g=\beta_b$ equal to $1$, the quadratic weights will oscillate around 0. The accuracy curves for ReLinear$^{sw}$-$l_1$ and ReLinear$^{sw}$-$l_2$ are shown in Figure \ref{fig:compare_sw}. It is observed that the $l_1$-norm is not good at stablizing the training in this task, while an appropriate $l_2$ norm, \textit{i.e.}, $\beta=0.9$ and $\beta=0.99$ manage to eliminate the large oscillation. The lowest error $8.21\%$ of ReLinear$^{sw}$ is achieved by the $l_2$ norm at $\beta_g=\beta_b=0.9$, which is worse than the lowest error $7.78\%$ of ReLinear$^{sg}$. 

\begin{figure}[htb]
\center{\includegraphics[width=\linewidth] {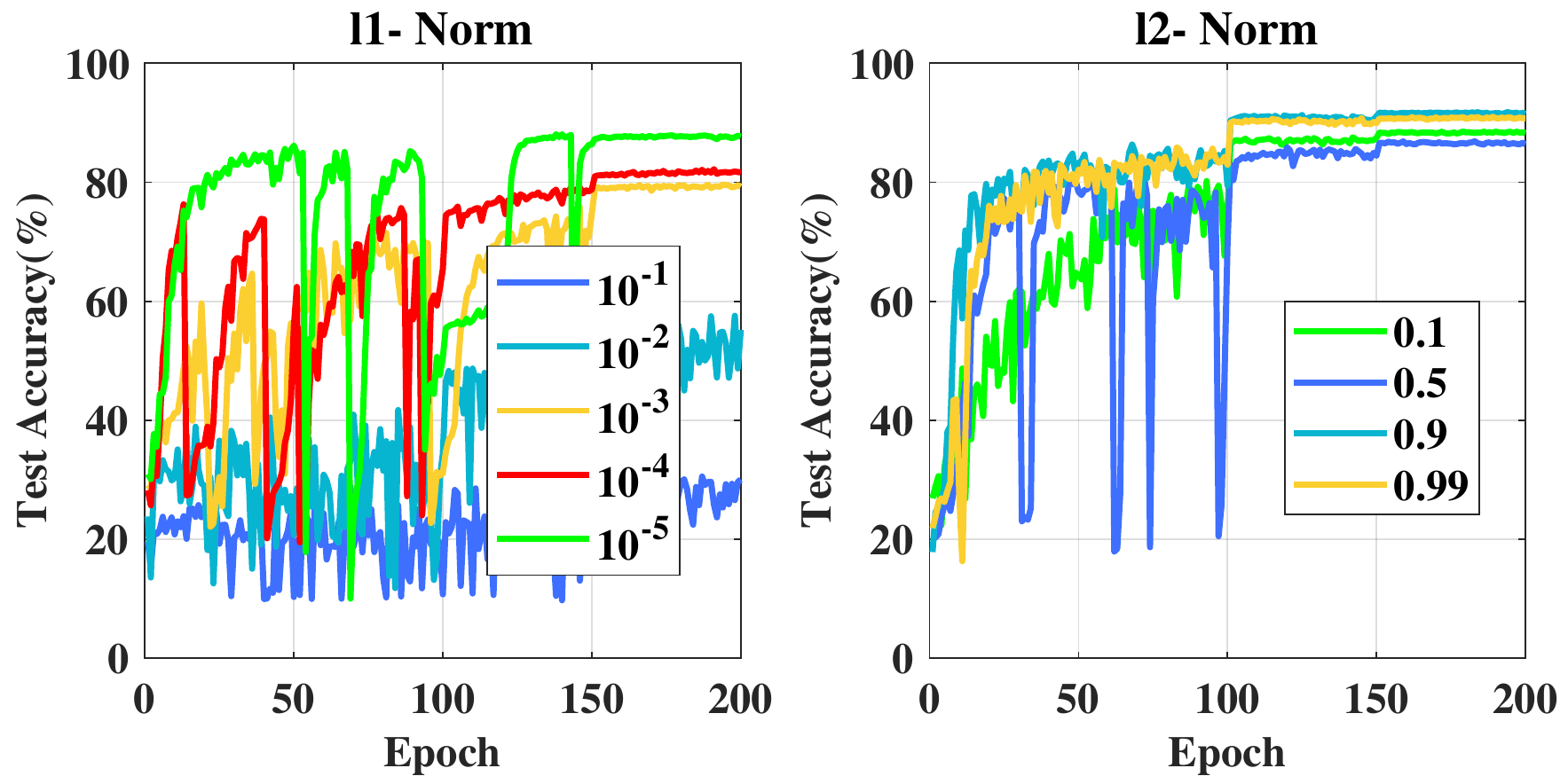}}
\caption{Left: accuracy curves obtained from different parameters ($\alpha_g, \alpha_b$) for ReLinear$^{sw}$-$l_1$. Right: accuracy curves obtained from different parameters ($\beta_g, \beta_b$) for ReLinear$^{sw}$-$l_2$.}
\label{fig:compare_sw}
\end{figure}

\begin{figure*}[htb]
\center{\includegraphics[width=0.7\linewidth] {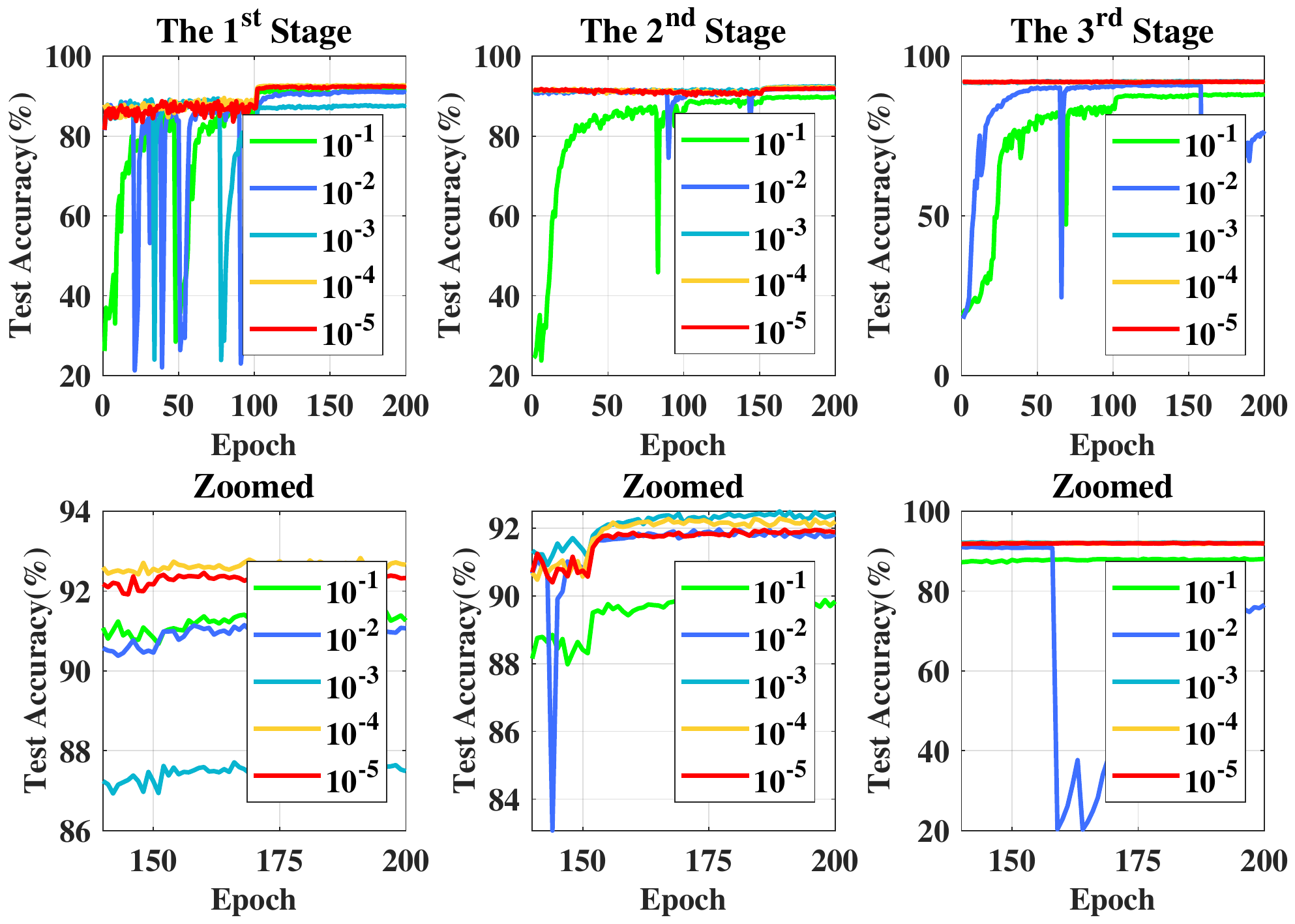}}
\caption{Accuracy curves from different learning rates by transferring weights from different stages.}
\label{fig:TL}
\vspace{-0.6 cm}
\end{figure*}

\textit{Weight Transfer.} As mentioned earlier, weight transfer can also be used to train a quadratic network. Because the training of a conventional ResNet has three stages (1-100 epochs, 101-150 epochs, 151-200 epochs), weight transfer also has three choices, corresponding to transferring weights from which stage. We evaluate all three choices. After transferring, the learning rate $\gamma_r$ will inherit the learning rate of the transferred model and then decay $1/10$ when the training moves to the next stage. Still, we set $\gamma_g=\gamma_b$ to $\{10^{-1}, 10^{-2}, 10^{-3}, 10^{-4}, 10^{-5}\}$ for a comprehensive analysis. Here we show the accuracy curves in Figure \ref{fig:TL}. There are two observations from Figure \ref{fig:TL}. First, transferred parameters can stabilize the training provided appropriate $(\gamma_g, \gamma_b)$. For the same $\gamma_g$, weights transferred from the later stage make the training more robust than those transferred from the earlier stage. This is because transferred parameters from the later stages have been good for the model. Thus, there is less need to optimize the quadratic terms, thereby avoiding the risk of explosion. The second highlight is that the best performance comes from weights transferred from the first stage, which suggests that the model can be improved if quadratic terms play a significant role. The lowest errors by transferring from three stages are $7.18\%$, $7.5\%$, and $7.75\%$. 

\textit{ReLinear+ReZero}
Especially, for training a residual quadratic network, the proposed ReLinear method can be integrated with the recent proposal called ReZero (residual with zero initialization, \cite{bachlechner2020rezero}), which dedicates to training a residual network by revising the residual propagation formula from $x^{h+1}=x^{h}+F(x)$ to $x^{h+1}=x^{h}+\zeta_h F(x)$, where $\zeta_h$ is initialized to be zero such that the network is an identity at the beginning. ReZero can not only speed up the training but also improve the performance of the model. Here, we evaluate the feasibility of training a quadratic residual network with ReLinear+ReZero. We adopt ReLinear$^{sg}$ and let $\gamma_g=\gamma_b$. We respectively set $\gamma_g, \gamma_b$ to $\{10^{-1}, 10^{-2}, 10^{-3}, 10^{-4}\}$. The results are shown in Figure \ref{fig:compare_relinear_rezero}. Comparing Figures \ref{fig:compare_relinear_rezero} and \ref{fig:compare_sg}, we surprisingly find that QResNet20 trained via ReLinear+Rezero has fewer oscillations. Previously, when $\gamma_g= \gamma_b=10^{-2}, 10^{-3}$, the accuracy curves from the only ReLinear still suffer unrest oscillations, while curves from the ReLinear+ReZero do not. 

\begin{figure}[htb]
\center{\includegraphics[width=\linewidth] {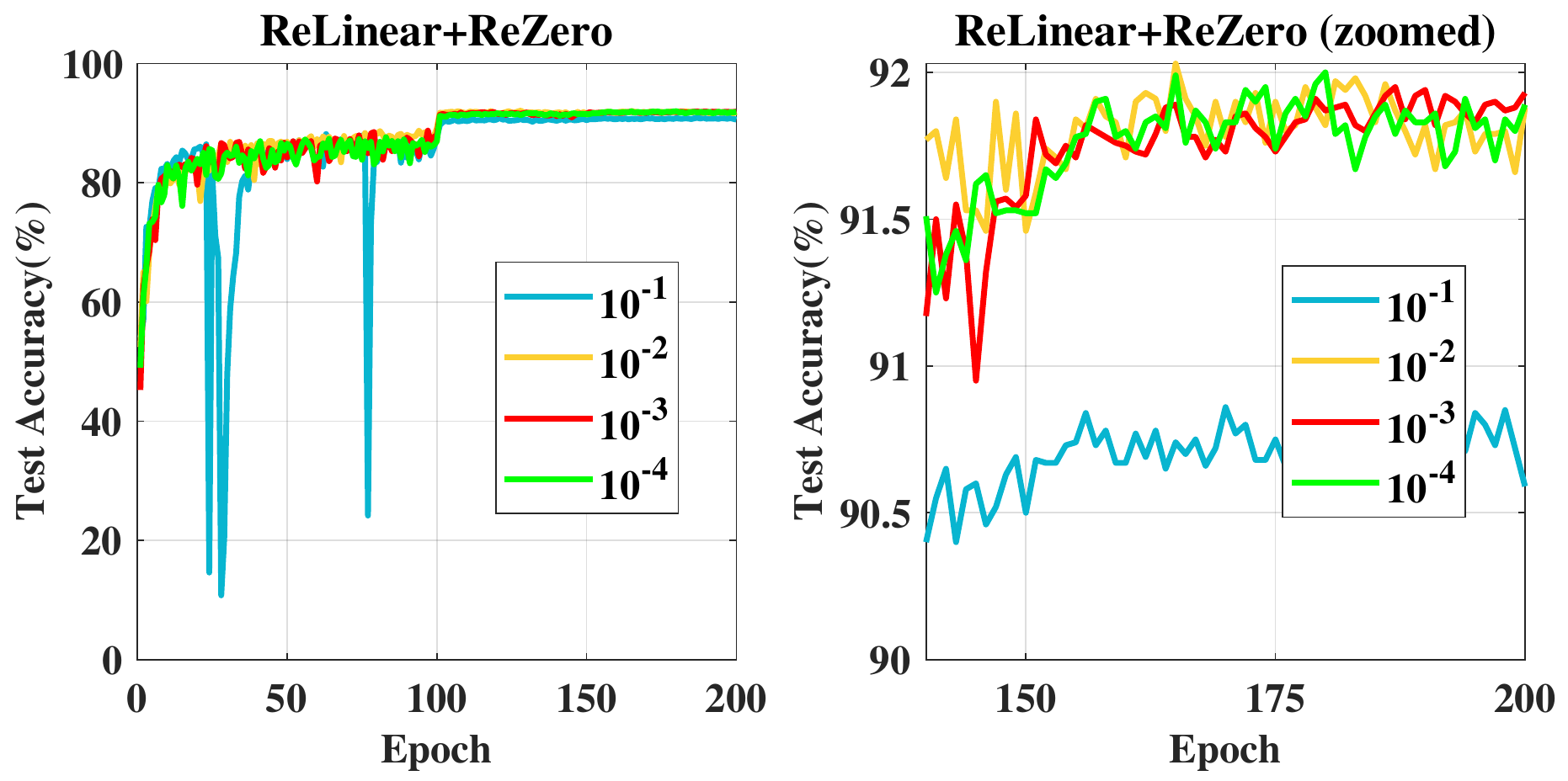}}
\caption{Accuracy curves obtained from different learning rates ($\gamma_g, \gamma_b$) for the ReLinear+ReZero method.}
\label{fig:compare_relinear_rezero}
\vspace{-0.2 cm}
\end{figure}


\subsubsection{Training Stability} 
Here, we compare our quadratic model with a model using quadratic activation on the CIFAR100 dataset. The VGG16 \cite{simonyan2014very} is used as the test bed. We obtain the VGG using quadratic activation by directly revising activation function, and we prototype the Quadratic-VGG16 by replacing the conventional neurons with quadratic neurons in VGG16. All training protocols of the quadratic VGG16 and the VGG16 using quadratic activation are the same as the authentic VGG16. We test three learning rates for VGG16 using quadratic activation: $0.01, 0.03, 0.05$. For the Quadratic-VGG16, we set the learning rate $\gamma_r$ to $0.1$ and $\gamma_g, \gamma_b$ to $10^{-6}$. The results are shown in Table \ref{qa}. We find that VGG16 using quadratic activation does not converge regardless of different learning rates, therefore, no errors can be reported. Indeed, direct training a network using quadratic activation suffers the magnitude as well due to the exponentially high degree. Preferably, this problem is overcome by the proposed training strategy in our quadratic network.

\begin{table}[htb]
 \centering
\caption{Error($\%$) of the quadratic VGG and VGG using quadratic activation on CIFAR100 validation set.}
 \begin{tabular}{ |c|c|  }
 \hline
 Network   & Error (\%) \\
  \hline
  VGG16(quadratic activation) l.r=0.05  &  no converge   \\
  VGG16(quadratic activation) l.r=0.03  &  no converge   \\  
  VGG16(quadratic activation) l.r=0.01    & no converge \\
  \hline
   Quadratic-VGG16   & 28.33 \\ 

 \hline
   \end{tabular}
 \vspace{1ex}
  \label{qa}
\end{table}

\subsection{Comparative Study}

Here, we validate the superiority of a quadratic network over a conventional one, with classification experiments on two data sets: CIFAR10 and ImageNet, and a depth estimation experiment. The quadratic network is implemented as a drop-in replacement for the conventional network; the only difference is the neuron type. Despite the straightforward replacement, aided by the proposed ReLinear, a quadratic network performs much better than its counterpart.

\textit{CIFAR10.} In this experiment, we systematically compare our QResNet with the ResNet. We follow the same protocol as the ResNet to train the QResNet, such as batch size and epoch number. As implied by our preceding experimental results that ReLinear$^{sg}$ is generally better than ReLinear$^{sw}$, we adopt ReLinear$^{sg}$. $\gamma_g=\gamma_b=10^{-4}, \gamma_g=\gamma_b=10^{-5}$, $\gamma_g=\gamma_b=10^{-5}$, $\gamma_g=\gamma_b=10^{-5}$ are set for QResNet20, QResNet32, QResNet56, and QResNet110, respectively. For all quadratic models, we test the weight transfer (the first stage) and a random initialization for their linear parts. We also implement ReLinear$^{sg}$+ReZero, where $\gamma_g=\gamma_b=10^{-5}, \gamma_g=\gamma_b=10^{-4}$, $\gamma_g=\gamma_b=10^{-4}$, $\gamma_g=\gamma_b=10^{-6}$ are set for QResNet20, QResNet32, QResNet56, and QResNet110, respectively. Table \ref{Cifar10_results} summarizes the results. Regardless of ways of initialization, all quadratic models are better than their counterparts, which is consistent with our analyses that quadratic neurons can improve model expressibility. Again, the improvement by quadratic networks is warranted because the employment of the proposed strategy makes the conventional model a special case of quadratic models. At least, a quadratic model will deliver the same as the conventional model. Furthermore, combined with the weight transfer, a quadratic network trained via ReLinear$^{sg}$ can surpass a conventional network by a larger margin, which suggests that a well-trained conventional model can effectively guide the training of a quadratic network.

There are two aspects in designing a network: structure and neurons. While it is widely recognized that going deep (structure design) can greatly increase the expressivity of a network, our analysis shows that adopting quadratic neurons (neuron design) can also enhance the expressivity. Both going deep and adopting quadratic neurons can increase the number of the network's parameters. However, which way can make a model gain better performance with fewer parameters depends on the architecture, training strategy, and so on. We do not claim that using quadratic neurons is always better than going deep. For example, as Table \ref{Cifar10_results} shows, for a conventional network, the errors of (ResNet20, ResNet56) are (8.75\%, 6.97\%), and their parameters are (0.27M, 0.86M). In contrast, the errors and parameters of (ResNet20, QResNet20) are (8.75\%, 7.17\%) and (0.27M, 0.81M), which means that for ResNet20, increasing depth is better than using quadratic neurons. On the other hand, the errors and parameters of (ResNet32, QResNet32) are (7.51\%, 6.38\%) and (0.46M, 1.39M), and the errors and parameters of (ResNet32, ResNet110) are (7.51\%, 6.61\%) and (0.46M, 1.7M). In this case, QResNet32 achieves a lower error with fewer parameters, suggesting that adopting quadratic neurons is better than going deep. The use of quadratic neurons offers a new direction in designing a network. Users can choose when to use new neurons and when to revise structures based on their needs.

In addition, we also compare QResNet with residual networks using three quadratic neuron designs \cite{tsapanos2018neurons, bu2021quadratic, xu2022quadralib, chrysos2021deep} on the CIFAR10 validation set, which are referred to as Parabolic-ResNet and PolyResNet, respectively. Please note that the neuron design in \cite{jiang2020nonlinear, mantini2021cqnn} are not compared because their parameter complexity $\mathcal{O}(n^2)$ is too high to be used in constructing a deep network. Both Parabolic-ResNet and PolyResNet are trained with the ReLinear algorithm because normally-trained Parabolic-ResNet and PolyResNet do not converge. Results are summarized in Table \ref{Cifar10_results}. It can be seen that as the depth increases, the performance of Parabolic-ResNet and PolyResNet does not necessarily improve. Moreover, with the same architecture, our quadratic models are consistently and substantially better than Parabolic-ResNet and PolyResNet. Since the key difference between our design and the designs in \cite{tsapanos2018neurons, chrysos2021deep} is the extra power term. Our results imply that the extra power term is indispensable.

\begin{table}[htb]
 \centering
\caption{Image classification by ResNet, QResNet, and ResNet using other quadratic neuron designs on the CIFAR10 validation set.}
 \begin{tabular}{ |c|c|c|c|  }
 \hline
 Network  &Paras & Time &  Error (\%) \\
  \hline
  ResNet20 \cite{he2016deep}  & 0.27M & 10.43 & 8.75  \\
  \hline
  ResNet32 \cite{he2016deep} & 0.46M & 14.39 & 7.51   \\
  \hline
  ResNet56 \cite{he2016deep}  & 0.86M & 21.61 & 6.97    \\
  \hline
  ResNet110 \cite{he2016deep} & 1.7M & 40.81 & 6.61   \\ 
  \hline
  ResNet1202\cite{he2016deep}  & 19.4M & 382.67 & 7.93   \\   
 
  \hline 
      Parabolic-ResNet20 \cite{tsapanos2018neurons}   & 0.54M &8.98 &  7.89\\
    Parabolic-ResNet32 \cite{tsapanos2018neurons}   & 0.54M &12.48 &   7.25\\
    Parabolic-ResNet56 \cite{tsapanos2018neurons}   & 0.54M &19.79 &  8.89\\  
        \hline
    PolyResNet20 \cite{chrysos2021deep}   & 0.54M &10.59 & 9.80 \\
    PolyResNet32 \cite{chrysos2021deep}  & 0.54M &15.29 & 7.62  \\
    PolyResNet56 \cite{chrysos2021deep}  & 0.54M &25.67 & 8.31 \\ 
  \hline 
    QResNet20 (r. i., ReLinear)  & 0.81M & 10.87 & 7.78 \\
    QResNet20 (r. i., ReLinear+ReZero) & 0.81M  & 11.18 &  7.97  \\
      QResNet20 (w. t., ReLinear) & 0.81M  & - &  \textbf{7.17}   \\
      \hline
    QResNet32 (r. i., ReLinear)   & 1.39M & 15.56 &  7.18  \\
    QResNet32 (r. i., ReLinear+ReZero)   & 1.39M & 17.07  &  6.90  \\
      QResNet32 (w. t., ReLinear) & 1.39M  & - &  \textbf{6.38}\\
      \hline
  QResNet56 (r. i., ReLinear) & 2.55M & 27.72  & 6.43   \\
  QResNet56 (r. i., ReLinear+ReZero) & 2.55M & 28.24 & 6.34   \\
  QResNet56 (w. t., ReLinear) & 2.55M & -  & \textbf{6.22}   \\

  \hline
  QResNet110 (r. i., ReLinear) & 5.1M & 55.18  &   6.36\\
  QResNet110 (r. i., ReLinear+ReZero) & 5.1M & 58.35  &   6.12\\
  QResNet110 (w. t., ReLinear) & 5.1M & -  & \textbf{5.44}   \\
  \hline  
   \end{tabular}
   
 \vspace{1ex}
 {Note: r. i. refers to random initialization, and w. t. refers to weight transfer. Time denotes the training time per epoch. \par}   
  \label{Cifar10_results}
\end{table}

Currently, an individual quadratic neuron has $3$ times parameters relative to an individual conventional neuron, which causes that a quadratic model is three times larger than the conventional model. To reduce the model complexity, we simplify the quadratic neuron by eradicating interaction terms, leading to a compact quadratic neuron: $q(\x)=\sigma(\x^\top \textbf{w}^{r}+(\x \odot \x)^\top \textbf{w}^{b}+c)$. The number of parameters in a compact quadratic neuron is twice that in a conventional neuron. Similarly, as a drop-in replacement, we implement QResNet20, QResNet32, and QResNet56 with compact quadratic neurons referred to as Compact-QResNet, and compare them with conventional models. The linear parts $\textbf{w}^r, c$ are initialized with a conventional ResNet. We also use ReLinear$^{sg}$ to train the Compact-QResNet and adopt $\gamma_b=10^{-4}$, $\gamma_b=10^{-5}$, and $\gamma_b=10^{-5}$ for the Compact-QResNet20, Compact-QResNet32, and Compact-QResNet56, respectively. It is seen from Table \ref{tab:compact_qnn} that Compact-QResNets are overall inferior to QResNets, but the margins between QResNet32/56 and Compact-QResNet32/56 are slight. It is noteworthy that both Compact-QResNet32 and QResNet32 are better than ResNet110, despite fewer parameters. 

\begin{table}[htb]
 \centering
\caption{Image classification error($\%$) by Compact-QResNet on CIFAR10 validation set.}
 \begin{tabular}{ |c|c|c|  }
 \hline
 Network   & Params &  Error (\%) \\
  \hline
  QResNet20 & 0.81M  &  \textbf{7.17}   \\

  Compact-QResNet20 & 0.54M    & 7.76\\
  
  \hline
  QResNet32 & 1.39M  &  \textbf{6.38}\\

  Compact-QResNet32 & 0.92M  &   6.56 \\
  \hline
  
    QResNet56 & 2.55M  & \textbf{6.22}   \\
  Compact QResNet56 & 1.92M  & 6.30  \\
 \hline
   \end{tabular}
 \vspace{1ex}
  \label{tab:compact_qnn}
\end{table}

Furthermore, we would like to underscore that the reason for improvements produced by quadratic networks is not increasing the number of parameters in a brute-force manner but facilitating the enhanced feature representation. The enhanced feature representation in some circumstances can also enjoy model efficiency. For example, purely stacking layers does not definitely facilitate the model performance, evidenced by the fact the ResNet1202 is worse than ResNet34, ResNet56, and ResNet110. Moreover, we implement the ResNet32 and ResNet56 with $1.5, 2, 2.5$ times channel numbers of the original. As Table \ref{tab:efficiency} shows, increasing channel numbers of the ResNets does not produce as much gain as using quadratic models with the approximately similar model size. We think that this is because a widened network is just a combination of more linear operations in the same layer, which does not change the type of representation too much.

\begin{table}[htb]
 \centering
\caption{Image classification error($\%$) by ResNets with increased channel numbers on CIFAR10 validation set..}
 \begin{tabular}{ |c|c|c|c|  }
 \hline
 Network & Channel Number  & Params &  Error (\%) \\
\hline
  ResNet32 (16) & $\times 1$ & 0.46M  &  7.51\\

  ResNet32 (24) & $\times 1.5$ & 1.04M  &  6.45  \\

  ResNet32 (32) & $\times 2$ & 1.84M  &  6.75 \\

  ResNet32 (40) & $\times 2.5$ & 2.88M  &  7.44  \\
  QResNet32 & - & 1.39M  &  \textbf{6.38} \\
 \hline
  ResNet56 (16) & $\times 1$ & 0.86M  &  6.97\\

  ResNet56 (24) & $\times 1.5$ & 1.92M  &  6.84  \\

  ResNet56 (32) & $\times 2$ & 3.40M  &  \textbf{6.12}  \\

  ResNet56 (40) & $\times 2.5$ & 5.31M  &  6.58 \\
  
  QResNet56 & - & 2.55M  & 6.22   \\
 \hline

   \end{tabular}
 \vspace{1ex}
  \label{tab:efficiency}
\end{table}

\textit{ImageNet.} Here, we confirm the superior model expressivity of quadratic networks with experiments on ImageNet. The ImageNet dataset \cite{deng2009imagenet}
is made of $1.2$ million images for training and $50,000$ images for validation, which are from $1,000$ classes. For model configurations, we follow those in the ResNet paper \cite{he2016deep}. We set the batch size to 256, the initial learning rate to 0.1, the weight decay to 0.0001, and the momentum to 0.9. For ReLinear$^{sg}$, we set $\gamma_b$ and $\gamma_r$ to $10^{-5}$. We adopt the standard 10-crop validation. As seen in Table \ref{ImageNetResults}, similar to what we observed in CIFAR10 experiments, direct replacement with quadratic neurons can promote the performance, which confirms that the quadratic network is more expressive than the conventional network. For example, QResNet32 is better than ResNet32 with a considerable margin ($>1\%$). In addition, the performance of QResNet32 is also better than Poly-QResNet32 and Para-QResNet32, which validates that our design of quadratic networks is competitive.

\begin{table}[htb]
 \centering
\caption{Image classification by different networks on the ImageNet validation set.}
\scalebox{0.9}{
 \begin{tabular}{ |l|c|l|c|  }
 \hline
 Network  &  Error($\%$) & Network  &  Error($\%$) \\
  \hline
  plain-20 &  27.94 & plain-32 & 28.54\\
  \hline
  ResNet20   & 27.88 & ResNet32 & 25.03 \\
  \hline
  \fl{Parabolic-ResNet20 \cite{tsapanos2018neurons}} & 28.22 & \fl{Parabolic-ResNet32 \cite{tsapanos2018neurons}}  & 24.64\\
  \hline
  \fl{Poly-ResNet20 \cite{chrysos2021deep}}   & 27.82 & \fl{Poly-ResNet32 \cite{chrysos2021deep}} & 24.63\\
  \hline
  QResNet20 &  \textbf{27.67} & QResNet32  &  \textbf{24.01}\\
  \hline
   \end{tabular}}
   
 \vspace{1ex}
 {The errors of ResNets are reported by the official implementation, while others are results of our implementation. \par}
 
 \label{ImageNetResults}
 \vspace{-2mm}
\end{table}

\textit{KITTI.} Depth estimation \cite{xu2022multi} has been widely applied in many applications, \textit{e.g.}, assisting agents to reconstruct the surrounding environment and estimating horizontal boundaries or location of the vanishing point to understand a given scene. Here, we validate the effectiveness of the proposed ReLinear strategy on the task of depth estimation. 

The dataset is KITTI which captures various road
environments in autonomous driving scenarios by the paired raw LiDaR scans and RGB images.
The images are of $1242\times375$ pixels.
With the split strategy \cite{eigen2014depth} for the KITTI dataset, the training set comprises 23,488 images from 32
scenes, while the test set contains 697 images from 29 scenes.

The baseline method is LapDepth, a well-performed benchmark in the field of depth estimation. The LapDepth uses ResNext101 as the pre-trained encoder and the Laplacian pyramidal residual structure as the decoder. Their idea is to leverage the ability of the Laplacian pyramid in highlighting the differences across different scale spaces and object boundaries. The depth map is progressively restored from coarse to fine by combining the depth residuals at each pyramid level, thereby overcoming ambiguities at the depth boundary. 

We modify the LapDepth model by just replacing the conventional neurons in ResNext101 with quadratic neurons in the encoder part. The resultant model is referred to as Quadratic-LapDepth.
We respect the parameter configuration in the original paper \cite{song2021monocular}. All the parameters in the Laplacian pyramidal decoder are initialized by Kaiming initialization \cite{he2015delving}. The model is trained
for 50 epochs from scratch using the
AdamW optimizer \cite{loshchilov2017decoupled} whose power and momentum are cast to 0.9 and 0.999, respectively. The batch size is 16. The weight decaying rate is 0.0005 for ResNext101 and zero for the Laplacian pyramidal decoder. The schedule for the learning rate is the polynomial decay with the power of 0.5, the initial value of $10^{-4}$, and the terminating value of $10^{-5}$.

\begin{table}[htb]
 \centering
\caption{Quantatitive Evaluation on the KITTI Dataset.}
 \begin{tabular}{ |c|c|c|c|  }
 \hline
 Model  & $\mathbf{S q} ~\mathbf{Rel}$ &  $\mathbf{R M S E}$  \\
  \hline
  LapDepth (reported officially)    & 0.212 & 0.446 \\
  \hline
  LapDepth (our replication)    & 0.214 & 0.458 \\
  \hline
  Quadratic-LapDepth & 0.200  & 0.427   \\
  \hline
   \end{tabular} 
  \label{KITTI_results}
\end{table}

The widely-used evaluation metrics in the field of depth estimation are 
\begin{equation}
\begin{aligned}
 & \mathbf{S q} ~\mathbf{Rel}=\frac{1}{T} \sum_{y \in T}\left\|y-y^{*}\right\|^{2} / y^{*}, \\
& \mathbf{R M S E}=\sqrt{\frac{1}{T} \sum_{y \in T}\left\|y-y^{*}\right\|^{2}} ,
\end{aligned}
\end{equation}
where $y^*$ and $y$ are the ground truth and the predicted depth map, respectively, and $T$ is the total number of effective pixels in the ground truth.

The experimental results are shown in Table \ref{KITTI_results}. There is a gap between the performance of our replication and what was officially reported of the LapDepth model. Regardless, the direct replacement with quadratic neurons in the encoder can substantially reduce the error, which confirms the power of quadratic neurons and the effectiveness of the proposed ReLinear on various learning tasks.

\section{DISCUSSION AND CONCLUSION}
Regarding the efficiency of quadratic networks, our opinion is that there must exist tasks that are suitable for quadratic networks such that a quadratic network is more efficient than a conventional one. Heuristically, there should be no universally superior type of neurons that can solve all machine learning tasks efficiently. For tasks whose quadratic or nonlinear features are significant, quadratic neurons should show efficiency, since it is easier to use quadratic neurons to learn quadratic or nonlinear features than conventional neurons. Previously, our group \cite{fan2020universal} proved a theorem that given the network with only one hidden layer, there exists a function that a quadratic network can approximate with a polynomial number of neurons, while a conventional network can only achieve the same level approximation with an exponential number of neurons. This theorem theoretically supports our opinion.

In this article, we have theoretically demonstrated the superior expressivity of a quadratic network over either the popular deep learning model or such a conventional model with quadratic activation. Then, we have proposed an effective and efficient ReLinear strategy for training a quadratic network, thereby improving its performance in various machine learning tasks. Finally, we have performed extensive experiments to corroborate our theoretical findings and confirm the practical gains with ReLinear. Future research includes up-scaling quadratic networks to solve more real-world problems and characterizing our quadratic approach in terms of its robustness, generalizability, and other properties.

\section*{Acknowledgement}

Dr. Fei Wang would like to acknowledge the support the AWS machine learning for a research award and the Google faculty research award. Dr. Rongjie Lai's research is supported in part by an NSF Career Award DMS–1752934 and DMS-2134168. Dr. Ge Wang would like to acknowledge the funding support from R01EB026646, R01CA233888, R01CA237267, R01HL151561, R21CA264772, and R01EB031102.

\bibliographystyle{ieeetr}
\bibliography{reference}

\begin{thebibliography}{10}

\bibitem{fuchs2021super}
F.~Fuchs, Y.~Song, E.~Kaufmann, D.~Scaramuzza, and P.~Duerr, ``Super-human
  performance in gran turismo sport using deep reinforcement learning,'' {\em
  IEEE Robotics and Automation Letters}, 2021.

\bibitem{you2019ct}
C.~You, G.~Li, Y.~Zhang, X.~Zhang, H.~Shan, M.~Li, S.~Ju, Z.~Zhao, Z.~Zhang,
  W.~Cong, {\em et~al.}, ``{CT} super-resolution {GAN} constrained by the
  identical, residual, and cycle learning ensemble (gan-circle),'' {\em IEEE
  Transactions on Medical Imaging}, vol.~39, no.~1, pp.~188--203, 2019.

\bibitem{he2016deep}
K.~He, X.~Zhang, S.~Ren, and J.~Sun, ``Deep residual learning for image
  recognition,'' in {\em Proceedings of the IEEE conference on computer vision
  and pattern recognition}, pp.~770--778, 2016.

\bibitem{fan2018sparse}
F.~Fan, D.~Wang, H.~Guo, Q.~Zhu, P.~Yan, G.~Wang, and H.~Yu, ``On a sparse
  shortcut topology of artificial neural networks,'' {\em arXiv preprint
  arXiv:1811.09003}, 2018.

\bibitem{liu2018progressive}
C.~Liu, B.~Zoph, M.~Neumann, J.~Shlens, W.~Hua, L.-J. Li, L.~Fei-Fei,
  A.~Yuille, J.~Huang, and K.~Murphy, ``Progressive neural architecture
  search,'' in {\em Proceedings of the European conference on computer vision
  (ECCV)}, pp.~19--34, 2018.

\bibitem{lecun2015deep}
Y.~LeCun, Y.~Bengio, and G.~Hinton, ``Deep learning,'' {\em nature}, vol.~521,
  no.~7553, pp.~436--444, 2015.

\bibitem{air}
S.~Chen, M.~Jiang, J.~Yang, and Q.~Zhao, ``Air: Attention with reasoning
  capability,'' in {\em ECCV}, 2020.

\bibitem{thivierge2008neural}
J.-P. Thivierge, ``Neural diversity creates a rich repertoire of brain
  activity,'' {\em Communicative \& integrative biology}, vol.~1, no.~2,
  pp.~188--189, 2008.

\bibitem{fan2019quadratic}
F.~Fan, H.~Shan, M.~K. Kalra, R.~Singh, G.~Qian, M.~Getzin, Y.~Teng, J.~Hahn,
  and G.~Wang, ``Quadratic autoencoder (q-ae) for low-dose ct denoising,'' {\em
  IEEE transactions on medical imaging}, vol.~39, no.~6, pp.~2035--2050, 2019.

\bibitem{fan2020universal}
F.~Fan, J.~Xiong, and G.~Wang, ``Universal approximation with quadratic deep
  networks,'' {\em Neural Networks}, vol.~124, pp.~383--392, 2020.

\bibitem{du2018power}
S.~Du and J.~Lee, ``On the power of over-parametrization in neural networks
  with quadratic activation,'' in {\em International Conference on Machine
  Learning}, pp.~1329--1338, PMLR, 2018.

\bibitem{mannelli2020optimization}
S.~S. Mannelli, E.~Vanden-Eijnden, and L.~Zdeborov{\'a}, ``Optimization and
  generalization of shallow neural networks with quadratic activation
  functions,'' {\em arXiv preprint arXiv:2006.15459}, 2020.

\bibitem{xu2021robust}
K.~Xu, H.~Bastani, and O.~Bastani, ``Robust generalization of quadratic neural
  networks via function identification,'' {\em arXiv preprint
  arXiv:2109.10935}, 2021.

\bibitem{bachlechner2020rezero}
T.~Bachlechner, B.~P. Majumder, H.~H. Mao, G.~W. Cottrell, and J.~McAuley,
  ``Rezero is all you need: Fast convergence at large depth,'' {\em arXiv
  preprint arXiv:2003.04887}, 2020.

\bibitem{ivakhnenko1971polynomial}
A.~G. Ivakhnenko, ``Polynomial theory of complex systems,'' {\em IEEE
  transactions on Systems, Man, and Cybernetics}, no.~4, pp.~364--378, 1971.

\bibitem{schmidhuber2015deep}
J.~Schmidhuber, ``Deep learning in neural networks: An overview,'' {\em Neural
  networks}, vol.~61, pp.~85--117, 2015.

\bibitem{poggio1975optimal}
T.~Poggio, ``On optimal nonlinear associative recall,'' {\em Biological
  Cybernetics}, vol.~19, no.~4, pp.~201--209, 1975.

\bibitem{giles1987learning}
C.~L. Giles and T.~Maxwell, ``Learning, invariance, and generalization in
  high-order neural networks,'' {\em Applied optics}, vol.~26, no.~23,
  pp.~4972--4978, 1987.

\bibitem{lippmann1989pattern}
R.~P. Lippmann, ``Pattern classification using neural networks,'' {\em IEEE
  communications magazine}, vol.~27, no.~11, pp.~47--50, 1989.

\bibitem{shin1991pi}
Y.~Shin and J.~Ghosh, ``The pi-sigma network: An efficient higher-order neural
  network for pattern classification and function approximation,'' in {\em
  IJCNN-91-Seattle international joint conference on neural networks}, vol.~1,
  pp.~13--18, IEEE, 1991.

\bibitem{milenkovic1996annealing}
S.~Milenkovic, Z.~Obradovic, and V.~Litovski, ``Annealing based dynamic
  learning in second-order neural networks,'' in {\em Proceedings of
  International Conference on Neural Networks (ICNN'96)}, vol.~1, pp.~458--463,
  IEEE, 1996.

\bibitem{zoumpourlis2017non}
G.~Zoumpourlis, A.~Doumanoglou, N.~Vretos, and P.~Daras, ``Non-linear
  convolution filters for cnn-based learning,'' in {\em Proceedings of the IEEE
  International Conference on Computer Vision}, pp.~4761--4769, 2017.

\bibitem{tsapanos2018neurons}
N.~Tsapanos, A.~Tefas, N.~Nikolaidis, and I.~Pitas, ``Neurons with paraboloid
  decision boundaries for improved neural network classification performance,''
  {\em IEEE transactions on neural networks and learning systems}, vol.~30,
  no.~1, pp.~284--294, 2018.

\bibitem{chrysos2021deep}
G.~Chrysos, S.~Moschoglou, G.~Bouritsas, J.~Deng, Y.~Panagakis, and S.~P.
  Zafeiriou, ``Deep polynomial neural networks,'' {\em IEEE Transactions on
  Pattern Analysis and Machine Intelligence}, 2021.

\bibitem{livni2014computational}
R.~Livni, S.~Shalev-Shwartz, and O.~Shamir, ``On the computational efficiency
  of training neural networks,'' {\em arXiv preprint arXiv:1410.1141}, 2014.

\bibitem{krotov2018dense}
D.~Krotov and J.~Hopfield, ``Dense associative memory is robust to adversarial
  inputs,'' {\em Neural computation}, vol.~30, no.~12, pp.~3151--3167, 2018.

\bibitem{redlapalli2003development}
S.~Redlapalli, M.~M. Gupta, and K.-Y. Song, ``Development of quadratic neural
  unit with applications to pattern classification,'' in {\em Fourth
  International Symposium on Uncertainty Modeling and Analysis, 2003. ISUMA
  2003.}, pp.~141--146, IEEE, 2003.

\bibitem{jiang2020nonlinear}
Y.~Jiang, F.~Yang, H.~Zhu, D.~Zhou, and X.~Zeng, ``Nonlinear cnn: improving
  cnns with quadratic convolutions,'' {\em Neural Computing and Applications},
  vol.~32, no.~12, pp.~8507--8516, 2020.

\bibitem{mantini2021cqnn}
P.~Mantini and S.~K. Shah, ``Cqnn: Convolutional quadratic neural networks,''
  in {\em 2020 25th International Conference on Pattern Recognition (ICPR)},
  pp.~9819--9826, IEEE, 2021.

\bibitem{bu2021quadratic}
J.~Bu and A.~Karpatne, ``Quadratic residual networks: A new class of neural
  networks for solving forward and inverse problems in physics involving
  pdes,'' in {\em Proceedings of the 2021 SIAM International Conference on Data
  Mining (SDM)}, pp.~675--683, SIAM, 2021.

\bibitem{xu2022quadralib}
Z.~Xu, F.~Yu, J.~Xiong, and X.~Chen, ``Quadralib: A performant quadratic neural
  network library for architecture optimization and design exploration,'' {\em
  Proceedings of Machine Learning and Systems}, vol.~4, pp.~503--514, 2022.

\bibitem{goyal2020improved}
M.~Goyal, R.~Goyal, and B.~Lall, ``Improved polynomial neural networks with
  normalised activations,'' in {\em 2020 International Joint Conference on
  Neural Networks (IJCNN)}, pp.~1--8, IEEE, 2020.

\bibitem{remmert1991fundamental}
R.~Remmert, ``The fundamental theorem of algebra,'' in {\em Numbers},
  pp.~97--122, Springer, 1991.

\bibitem{kileel2019expressive}
J.~Kileel, M.~Trager, and J.~Bruna, ``On the expressive power of deep
  polynomial neural networks,'' {\em Advances in Neural Information Processing
  Systems}, vol.~32, pp.~10310--10319, 2019.

\bibitem{jayakumar2019multiplicative}
S.~M. Jayakumar, W.~M. Czarnecki, J.~Menick, J.~Schwarz, J.~Rae, S.~Osindero,
  Y.~W. Teh, T.~Harley, and R.~Pascanu, ``Multiplicative interactions and where
  to find them,'' in {\em International Conference on Learning
  Representations}, 2019.

\bibitem{nguyen2019deep}
T.~Nguyen, A.~Kashani, T.~Ngo, and S.~Bordas, ``Deep neural network with
  high-order neuron for the prediction of foamed concrete strength,'' {\em
  Computer-Aided Civil and Infrastructure Engineering}, vol.~34, no.~4,
  pp.~316--332, 2019.

\bibitem{boyd1992defeating}
J.~P. Boyd, ``Defeating the runge phenomenon for equispaced polynomial
  interpolation via tikhonov regularization,'' {\em Applied Mathematics
  Letters}, vol.~5, no.~6, pp.~57--59, 1992.

\bibitem{birman1967piecewise}
M.~S. Birman and M.~Z. Solomyak, ``Piecewise-polynomial approximations of
  functions of the classes w\_p\^{}$\alpha$,'' {\em Matematicheskii Sbornik},
  vol.~115, no.~3, pp.~331--355, 1967.

\bibitem{hall1976optimal}
C.~A. Hall and W.~W. Meyer, ``Optimal error bounds for cubic spline
  interpolation,'' {\em Journal of Approximation Theory}, vol.~16, no.~2,
  pp.~105--122, 1976.

\bibitem{wang2013multivariate}
R.-H. Wang, {\em Multivariate spline functions and their applications},
  vol.~529.
\newblock Springer Science \& Business Media, 2013.

\bibitem{eisenbud2013commutative}
D.~Eisenbud, {\em Commutative algebra: with a view toward algebraic geometry},
  vol.~150.
\newblock Springer Science \& Business Media, 2013.

\bibitem{siegel2020approximation}
J.~W. Siegel and J.~Xu, ``Approximation rates for neural networks with general
  activation functions,'' {\em Neural Networks}, vol.~128, pp.~313--321, 2020.

\bibitem{kumar2017weight}
S.~K. Kumar, ``On weight initialization in deep neural networks,'' {\em arXiv
  preprint arXiv:1704.08863}, 2017.

\bibitem{chen2021pre}
H.~Chen, Y.~Wang, T.~Guo, C.~Xu, Y.~Deng, Z.~Liu, S.~Ma, C.~Xu, C.~Xu, and
  W.~Gao, ``Pre-trained image processing transformer,'' in {\em Proceedings of
  the IEEE/CVF Conference on Computer Vision and Pattern Recognition},
  pp.~12299--12310, 2021.

\bibitem{li2019convergence}
X.~Li and F.~Orabona, ``On the convergence of stochastic gradient descent with
  adaptive stepsizes,'' in {\em The 22nd International Conference on Artificial
  Intelligence and Statistics}, pp.~983--992, PMLR, 2019.

\bibitem{kingma2014adam}
D.~P. Kingma and J.~Ba, ``Adam: A method for stochastic optimization,'' {\em
  arXiv preprint arXiv:1412.6980}, 2014.

\bibitem{simonyan2014very}
K.~Simonyan and A.~Zisserman, ``Very deep convolutional networks for
  large-scale image recognition,'' {\em arXiv preprint arXiv:1409.1556}, 2014.

\bibitem{deng2009imagenet}
J.~Deng, W.~Dong, R.~Socher, L.-J. Li, K.~Li, and L.~Fei-Fei, ``Imagenet: A
  large-scale hierarchical image database,'' in {\em 2009 IEEE conference on
  computer vision and pattern recognition}, pp.~248--255, Ieee, 2009.

\bibitem{xu2022multi}
J.~Xu, X.~Liu, Y.~Bai, J.~Jiang, K.~Wang, X.~Chen, and X.~Ji, ``Multi-camera
  collaborative depth prediction via consistent structure estimation,'' in {\em
  Proceedings of the 30th ACM International Conference on Multimedia},
  pp.~2730--2738, 2022.

\bibitem{eigen2014depth}
D.~Eigen, C.~Puhrsch, and R.~Fergus, ``Depth map prediction from a single image
  using a multi-scale deep network,'' {\em Advances in neural information
  processing systems}, vol.~27, 2014.

\bibitem{song2021monocular}
M.~Song, S.~Lim, and W.~Kim, ``Monocular depth estimation using laplacian
  pyramid-based depth residuals,'' {\em IEEE transactions on circuits and
  systems for video technology}, vol.~31, no.~11, pp.~4381--4393, 2021.

\bibitem{he2015delving}
K.~He, X.~Zhang, S.~Ren, and J.~Sun, ``Delving deep into rectifiers: Surpassing
  human-level performance on imagenet classification,'' in {\em Proceedings of
  the IEEE international conference on computer vision}, pp.~1026--1034, 2015.

\bibitem{loshchilov2017decoupled}
I.~Loshchilov and F.~Hutter, ``Decoupled weight decay regularization,'' {\em
  arXiv preprint arXiv:1711.05101}, 2017.

\bibitem{alber1998projected}
Y.~I. Alber, A.~N. Iusem, and M.~V. Solodov, ``On the projected subgradient
  method for nonsmooth convex optimization in a hilbert space,'' {\em
  Mathematical Programming}, vol.~81, no.~1, pp.~23--35, 1998.

\end{thebibliography}
	

	



\appendix
Here, we analyze the convergence behavior of the ReLinear. Not only is the convergence result established but also the associated
convergence rate is derived. Our proofs are heavily based on \cite{li2019convergence}.

\textbf{Setting.} The problem of interest is to optimize the following
loss function:
\begin{equation}
    \underset{\x}{\min}~ \mathbb{E}_{\x} L(\bm{\Xi}, \x) ,
\end{equation}
where $\bm{\Xi}$ is a collection of network parameters, and $\x$ is the data. The stochastic gradient descent (SGD) with the ReLinear algorithm at the $t$-th step is as follows:

i) select a set of samples from the dataset, denoted as $\X_t = \{\x_m\}_{m=1}^{b_t}$ and compute the gradient accordingly
\begin{equation}
    \boldsymbol{g}\left(\boldsymbol{\Xi}_{t}, \mathbf{X}_{t}\right)=\frac{1}{b_{t}} \sum_{i=1}^{b_{t}} \nabla L\left(\boldsymbol{\Xi}_{t}, \boldsymbol{x}^{i}\right).
\end{equation}

ii) update the network parameters,
\begin{equation}
  \Xi_{t+1}=\Xi_{t}-\boldsymbol{\eta}_{t} \boldsymbol{g}\left(\Xi_{t}, \mathbf{X}_{t}\right),
\end{equation}
where $\boldsymbol{\eta}_{t}$ is a diagonal matrix whose diagonal entries are the
learning rates of different parameters. Such a setting complies
with the ReLinear that assigns $(\bm{W}^r, \bm{b}^r)$ and $(\bm{W}^g, \bm{b}^g, \W^b, c)$ to different learning rates. The learning rate can be either fixed or
diminished. Without loss of generality, let $\boldsymbol{\eta}_{t}$  diminish along the
training step:
\begin{equation}
    \boldsymbol{\eta}_{t}=\frac{\beta}{\alpha+t} \operatorname{diag}\left(\left[\mathbf{1}_{K_{1}}, \gamma \cdot \mathbf{1}_{K_{2}}\right]\right),
    \label{adaptive}
\end{equation}
where $\mathbf{1}_{K}$ is an all-ones vector of the length $K$, $\alpha, \beta, \gamma$ are
auxiliary parameters, $K_1$ and $K_2$ represent the number of trainable parameters in $(\bm{W}^r, \bm{b}^r)$ and $(\bm{W}^g, \bm{b}^g, \W^b, c)$, respectively. $\boldsymbol{\eta}_{t}$
fulfills that $\sum_j^\infty (\boldsymbol{\eta}_{t})_{jj}$ diverges and $\sum_j^\infty (\boldsymbol{\eta}_{t})_{jj}^2$
converges.

\textbf{Assumptions}. We adopt the assumptions that are often made in proving convergence of optimization algorithms.

\begin{enumerate}[label=(\textbf{H{{\arabic*}}})]
\item $L$ is $M$ -smooth, i.e.,  $L$ is differentiable and  $\| \nabla L\left(\boldsymbol{\Xi}_{1}\right)- \nabla L\left(\boldsymbol{\Xi}_{2}\right)\|\leq M\| \boldsymbol{\Xi}_{1}-\boldsymbol{\Xi}_{2}\|, \quad \forall \boldsymbol{\Xi}_{1}, \boldsymbol{\Xi}_{2} $. Furthermore, $M$-smoothness implies that
\begin{equation}
L\left(\boldsymbol{\Xi}_{2}\right) \leq L\left(\boldsymbol{\Xi}_{1}\right)+\left\langle\nabla L\left(\boldsymbol{\Xi}_{1}\right), \boldsymbol{\Xi}_{2}-\boldsymbol{\Xi}_{1}\right\rangle+\frac{M}{2}\left\|\boldsymbol{\Xi}_{2}-\boldsymbol{\Xi}_{1}\right\|^{2}
\end{equation}

\item $L$  is  $\Theta$ -Lipschitz, i.e.,  $\left|L\left(\boldsymbol{\Xi}_{1}\right)-L\left(\boldsymbol{\Xi}_{2}\right)\right| \leq \Theta \| \boldsymbol{\Xi}_{1}-   \boldsymbol{\Xi}_{2} \|, \forall \boldsymbol{\Xi}_{1}, \boldsymbol{\Xi}_{2} $.

\item The expectation of the stochastic gradient is equal to the true gradient:  $\mathbb{E}_{\mathbf{X}}[\boldsymbol{g}(\boldsymbol{\Xi}, \mathbf{X})]=\nabla L(\boldsymbol{\Xi}) , \forall \boldsymbol{\Xi}$.

\item The noise in the stochastic gradient has bounded support, i.e.,  $\left\|\boldsymbol{g}\left(\boldsymbol{\Xi}_{t}, \mathbf{X}_{t}\right)-\nabla L\left(\boldsymbol{\Xi}_{t}\right)\right\| \leq S, \forall \boldsymbol{\Xi}_{t}, \mathbf{X}_{t} $. Furthermore, we have
\begin{equation}
\begin{aligned}
&\left\|\boldsymbol{\eta}_{t}\left(\boldsymbol{g}\left(\boldsymbol{\Xi}_{t}, \mathbf{X}_{t}\right)-\nabla L\left(\boldsymbol{\Xi}_{t}\right)\right)\right\|^{2} \\
=&\left\|\boldsymbol{\eta}_{t} \boldsymbol{g}\left(\boldsymbol{\Xi}_{t}, \mathbf{X}_{t}\right)\right\|^{2}+\left\|\boldsymbol{\eta}_{t} \nabla L\left(\boldsymbol{\Xi}_{t}\right)\right\|^{2}\\
&-2\left\langle\boldsymbol{\eta}_{t} \boldsymbol{g}\left(\boldsymbol{\Xi}_{t}, \mathbf{X}_{t}\right), \boldsymbol{\eta}_{t} \nabla L\left(\boldsymbol{\Xi}_{t}\right)\right\rangle \\
\leq &\left\|\boldsymbol{\eta}_{t}\right\|^{2}\left\|\left(\boldsymbol{g}\left(\boldsymbol{\Xi}_{t}, \mathbf{X}_{t}\right)-\nabla L\left(\boldsymbol{\Xi}_{t}\right)\right)\right\|^{2}.
\end{aligned}
\end{equation}
Taking the expectation of the above equation leads to
\begin{equation}
\mathbb{E}\left[\left\|\boldsymbol{\eta}_{t} \boldsymbol{g}\left(\boldsymbol{\Xi}_{t}, \mathbf{X}_{t}\right)\right\|^{2}\right] \leq \mathbb{E}\left[\left\|\boldsymbol{\eta}_{t} \nabla L\left(\boldsymbol{\Xi}_{t}\right)\right\|^{2}\right]+\left\|\boldsymbol{\eta}_{t}\right\|^{2} S^{2}.
\label{eqn:gradient_bound}
\end{equation}
\end{enumerate}

\begin{lemma} \cite{alber1998projected}
Let  $\left(a_{t}\right)_{t \geq 1}$  and  $\left(b_{t}\right)_{t \geq 1}$  be two non-negative sequences. Assume that  $\sum_{t=1}^{\infty} a_{t} b_{t}$  converges and  $\sum_{t=1}^{\infty} a_{t}$ diverges, and there exists  $K \geq 0$ such that  $\left|b_{t+1}-b_{t}\right| \leq K a_{t}$, then  $b_{t}$ converges to 0.
\label{lemma:series}
\end{lemma}
\begin{proof}
The proof can be referred to Lemma 2 of \cite{li2019convergence}.
\end{proof} 

\begin{lemma}
Assume  (\textbf{H1}, \textbf{H3}, \textbf{H4}) , the iterates of  SGD  with a learning rate diagonal matrix  $\boldsymbol{\eta}_{t}$  satisfying  Eq. \eqref{adaptive}  and  $0 \leq \left(\boldsymbol{\eta}_{t}\right)_{j j} \leq \frac{1}{M}$, it holds that
\begin{equation}
\begin{aligned}
& \mathbb{E}[\sum_{t=1}^{T}\left\langle\nabla L\left(\boldsymbol{\Xi}_{t}\right), \boldsymbol{\eta}_{t} \nabla L\left(\boldsymbol{\Xi}_{t}\right)\right\rangle] \\
\leq & 2(\mathbb{E}\left[L\left(\boldsymbol{\Xi}_{1}\right)-L\left(\boldsymbol{\Xi}_{T+1}\right)\right]+\frac{M S^{2}}{2} \sum_{t=1}^{T}\left\|\boldsymbol{\eta}_{t}\right\|^{2}).
\end{aligned}
\end{equation}
\label{lemma:key}
\end{lemma}

\begin{proof}
According to (\textbf{H1}), we have
\begin{equation}
\begin{aligned}
& L\left(\boldsymbol{\Xi}_{t+1}\right) \\
\leq & L\left(\boldsymbol{\Xi}_{t}\right)+\left\langle\nabla L\left(\boldsymbol{\Xi}_{t}\right), \boldsymbol{\Xi}_{t+1}-\boldsymbol{\Xi}_{t}\right\rangle+\frac{M}{2}\left\|\boldsymbol{\Xi}_{t+1}-\boldsymbol{\Xi}_{t}\right\|^{2} \\
\leq & L\left(\boldsymbol{\Xi}_{t}\right)-\left\langle\nabla L\left(\boldsymbol{\Xi}_{t}\right), \boldsymbol{\eta}_{t} \boldsymbol{g}\left(\boldsymbol{\Xi}_{t}, \mathbf{X}_{t}\right)\right\rangle+\frac{M}{2}\left\|\boldsymbol{\Xi}_{t+1}-\boldsymbol{\Xi}_{t}\right\|^{2} \\
=& L\left(\boldsymbol{\Xi}_{t}\right)+\left\langle\nabla L\left(\boldsymbol{\Xi}_{t}\right), \boldsymbol{\eta}_{t}\left(\nabla L\left(\boldsymbol{\Xi}_{t}\right)-g\left(\boldsymbol{\Xi}_{t}, \mathbf{X}_{t}\right)\right)\right\rangle \\
&-\left\langle\nabla L\left(\boldsymbol{\Xi}_{t}\right), \boldsymbol{\eta}_{t} \nabla L\left(\boldsymbol{\Xi}_{t}\right)\right\rangle+\frac{M}{2}\left\|\boldsymbol{\eta}_{t} \boldsymbol{g}\left(\boldsymbol{\Xi}_{t}, \boldsymbol{x}_{t}\right)\right\|^{2}
\end{aligned}
\label{eqn:basic}
\end{equation}

Taking the conditional expectation with respect to  $\mathbf{X}_{1}, \mathbf{X}_{2}, \cdots, \mathbf{X}_{t-1}$ , due to (\textbf{H3}), we have
\begin{equation}
    \mathbb{E}_{\mathbf{X}_{t}}\left[g\left(\boldsymbol{\Xi}_{t}, \mathbf{X}_{t}\right)\right]=\nabla L\left(\boldsymbol{\Xi}_{t}\right).
\end{equation}

Then, taking the expectation for Eq. \eqref{eqn:basic} and applying Eq. \eqref{eqn:gradient_bound}, we have
\begin{equation}
    \begin{aligned}
& \mathbb{E}\left[\left\langle\nabla L\left(\boldsymbol{\Xi}_{t}\right), \boldsymbol{\eta}_{t} \nabla L\left(\boldsymbol{\Xi}_{t}\right)\right\rangle\right] \\
\leq & \mathbb{E}\left[L\left(\boldsymbol{\Xi}_{t}\right)-L\left(\boldsymbol{\Xi}_{t+1}\right)\right]+\frac{M}{2} \mathbb{E}[\left\|\boldsymbol{\eta}_{t} g\left(\boldsymbol{\Xi}_{t}, \boldsymbol{x}_{t}\right)\right\|^{2}] \\
\leq &\mathbb{E}\left[L\left(\boldsymbol{\Xi}_{t}\right)-L\left(\boldsymbol{\Xi}_{t+1}\right)\right]+\frac{M}{2}(\mathbb{E}[\left\|\boldsymbol{\eta}_{t} \nabla L\left(\boldsymbol{\Xi}_{t}\right)\right\|^{2}]+\left\|\boldsymbol{\eta}_{t}\right\|^{2} S^{2}) .
\end{aligned}
\end{equation}

Moving  $\frac{M}{2} \mathbb{E}[\left\|\boldsymbol{\eta}_{t} \nabla L\left(\boldsymbol{\Xi}_{t}\right)\right\|^{2}]$  to the left leads to
\begin{equation}
\begin{aligned}
& \mathbb{E}[\langle\nabla L\left(\boldsymbol{\Xi}_{t}\right),(\boldsymbol{\eta}_{t}-\frac{M}{2} \boldsymbol{\eta}_{t}^{2}) \nabla L\left(\boldsymbol{\Xi}_{t}\right)\rangle] \\
\leq & \mathbb{E}\left[L\left(\boldsymbol{\Xi}_{t}\right)-L\left(\boldsymbol{\Xi}_{t+1}\right)\right]+\frac{M S^{2}\left\|\boldsymbol{\eta}_{t}\right\|^{2}}{2}.
\end{aligned}
\end{equation}
Because $0 \leq\left(\boldsymbol{\eta}_{t}\right)_{j j} \leq \frac{1}{M}$, we have $\boldsymbol{\eta}_{t}-\frac{M}{2} \boldsymbol{\eta}_{t}^{2} \geq \frac{1}{2} \boldsymbol{\eta}_{t}$. Then,
\begin{equation}
\begin{aligned}
& \mathbb{E}[\langle\nabla L\left(\boldsymbol{\Xi}_{t}\right), \frac{\boldsymbol{\eta}_{t}}{2} \nabla L\left(\boldsymbol{\Xi}_{t}\right)\rangle] \\
\leq & \mathbb{E}\large[\langle\nabla L\left(\boldsymbol{\Xi}_{t}\right),(\boldsymbol{\eta}_{t}-\frac{M}{2} \boldsymbol{\eta}_{t}^{2}) \nabla L\left(\boldsymbol{\Xi}_{t}\right)\rangle\large] \\
\leq & \mathbb{E}\left[L\left(\boldsymbol{\Xi}_{t}\right)-L\left(\boldsymbol{\Xi}_{t+1}\right)\right]+\frac{M S^{2}\left\|\boldsymbol{\eta}_{t}\right\|^{2}}{2}
\end{aligned}
\end{equation}
Summing the above from  $t=1$ to $T$, we conclude the proof.
\end{proof}

\begin{prop}[Convergence Behavior] Assume (\textbf{H1},\textbf{H2},\textbf{H3},\textbf{H4}). Suppose that the learning rates $\boldsymbol{\eta}_{t}$  are given by Eq. \eqref{adaptive} and  $0 \leq\left(\boldsymbol{\eta}_{t}\right)_{j j} \leq \frac{1}{M}$, then the gradients of SGD converge to zero almost surely.
\end{prop}

\begin{proof}
Use Lemma \ref{lemma:key} and take the limit for  $T \rightarrow \infty $,
\begin{equation}
\begin{aligned}
& \mathbb{E}\large[\sum_{t=1}^{\infty}\left\langle\nabla L\left(\boldsymbol{\Xi}_{t}\right), \boldsymbol{\eta}_{t} \nabla L\left(\boldsymbol{\Xi}_{t}\right)\right\rangle\large] \\
\leq & 2\large(\mathbb{E}\left[L\left(\boldsymbol{\Xi}_{1}\right)-L^{*}\right]+\frac{M S^{2}}{2} \sum_{t=1}^{\infty}\left\|\boldsymbol{\eta}_{t}\right\|^{2}\large)<\infty.
\end{aligned}
\end{equation}
Since if $E[Y]<\infty$ , with probability  1, $Y<\infty$, we have
\begin{equation}
    \sum_{t=1}^{\infty}\left\langle\nabla L\left(\boldsymbol{\Xi}_{t}\right), \boldsymbol{\eta}_{t} \nabla L\left(\boldsymbol{\Xi}_{t}\right)\right\rangle<\infty,
\end{equation}
which naturally gives rise to
\begin{equation}
    \sum_{t=1}^{\infty}\left(\boldsymbol{\eta}_{t}\right)_{j j}|\nabla L\left(\boldsymbol{\Xi}_{t}\right)_{j}|^{2}<\infty .
    \label{eqn:bound1}
\end{equation}

Using the fact that  $L$  is  $\Theta$-Lipschitz and $M$-smooth,
\begin{equation}
\begin{aligned}
&|(\left(\nabla L\left(\boldsymbol{\Xi}_{t+1}\right)\right)_{j})^{2}-\left(\left(\nabla L(\boldsymbol{\Xi}_{t}\right)\right)_{j})^{2}| \\
=&(\left(\nabla L\left(\boldsymbol{\Xi}_{t+1}\right)\right)_{j}+\left(\nabla L\left(\boldsymbol{\Xi}_{t}\right)\right)_{j})|\left(\nabla L\left(\boldsymbol{\Xi}_{t+1}\right)\right)_{j}-\left(\nabla L\left(\boldsymbol{\Xi}_{t}\right)\right)_{j}| \\
\leq & 2 \Theta|\left(\nabla L\left(\boldsymbol{\Xi}_{t+1}\right)\right)_{j}-\left(\nabla L\left(\boldsymbol{\Xi}_{t}\right)\right)_{j}| \\
\leq & 2 \Theta M\left\|\boldsymbol{\Xi}_{t+1}-\boldsymbol{\Xi}_{t}\right\| \\
=& 2 \Theta M\left\|\boldsymbol{\eta}_{t} \boldsymbol{g}\left(\boldsymbol{\Xi}_{t}, \mathbf{X}_{t}\right)\right\| \leq 2 \Theta M\left\|\boldsymbol{\eta}_{t}\right\|\left\|\boldsymbol{g}\left(\boldsymbol{\Xi}_{t}, \mathbf{X}_{t}\right)\right\| \\
\leq & 2 \Theta M(\Theta+S) \frac{\sqrt{K_{1}+K_{2}}}{\gamma}\left(\boldsymbol{\eta}_{t}\right)_{j j},
\end{aligned}
\label{eqn:bound2}
\end{equation}
where the first inequality is due to that $L$ is $\Theta$ -Lipschitz, i.e.,  $|\nabla L(\boldsymbol{\Xi})| \leq \Theta$, and the last inequality is because
\begin{equation}
    \left\|\boldsymbol{\eta}_{t}\right\| \leq \sqrt{K_{1}+K_{2}}\left(\boldsymbol{\eta}_{t}\right)_{\max } \leq \frac{\sqrt{K_{1}+K_{2}}}{\gamma}\left(\boldsymbol{\eta}_{t}\right)_{j j}, \forall j.
\end{equation}
According to Lemma \ref{lemma:series}, combining Eqs. \eqref{eqn:bound1} and  \eqref{eqn:bound2}  as well as the fact that $\sum_{t}^{\infty}\left(\boldsymbol{\eta}_{t}\right)_{jj}$  diverges, we obtain that almost surely,
\begin{equation}
    \lim_{t \rightarrow \infty}(\left(\nabla L\left(\boldsymbol{\Xi}_{t}\right)\right)_{j})^{2}=0 .
\end{equation}

\end{proof}

\begin{prop}[Convergence Rate] Assume (\textbf{H1},\textbf{H2},\textbf{H3},\textbf{H4}). Suppose that the learning rate diagonal matrix  $\boldsymbol{\eta}_{t}$  is given by Eq. \eqref{adaptive} and  $0 \leq\left(\boldsymbol{\eta}_{t}\right)_{j j} \leq \frac{1}{M} $, then the iterates of SGD satisfy the following bound:
\begin{equation}
    \mathbb{E}[\min _{1 \leq t \leq T}\left\|\nabla L\left(\boldsymbol{\Xi}_{t}\right)\right\|] \leq T^{-1/2} \mathcal{O}(1).
\end{equation}
\end{prop}

\begin{proof}
From Proposition 3, we have
\begin{equation}
    \begin{aligned}
& \mathbb{E}[\sum_{t=1}^{T}\langle\nabla L(\boldsymbol{\Xi}_{t}), \boldsymbol{\eta}_{t} \nabla L\left(\boldsymbol{\Xi}_{t}\right)\rangle ] \\
\leq & 2(\mathbb{E}\left[L\left(\boldsymbol{\Xi}_{1}\right)-L\left(\boldsymbol{\Xi}_{T+1}\right)\right]+\frac{M S^{2}}{2} \sum_{t=1}^{T}\left\|\boldsymbol{\eta}_{t}\right\|^{2}) \\
\end{aligned}
\end{equation}

Let  $\Delta=\sum_{t=1}^{T}\left\|\nabla L\left(\boldsymbol{\Xi}_{t}\right)\right\|^{2} $, we can bound  $\mathbb{E}[\sum_{t=1}^{T}\left\langle\nabla L\left(\boldsymbol{\Xi}_{t}\right), \boldsymbol{\eta}_{t} \nabla L\left(\boldsymbol{\Xi}_{t}\right)\right\rangle]$  by
\begin{equation}
\begin{aligned}
& \mathbb{E}[\sum_{t=1}^{T}\left\langle\nabla L\left(\boldsymbol{\Xi}_{t}\right), \boldsymbol{\eta}_{t} \nabla L\left(\boldsymbol{\Xi}_{t}\right)\right\rangle] \\
\geq & \mathbb{E}\left[\eta_{T, \min } \Delta\right] \geq \eta_{T, \min }(\mathbb{E}[\sqrt{\Delta}])^{2},
\end{aligned}
\end{equation}
where the second inequality is Holder's inequality. Let  $A=   \frac{1}{\eta_{T, \min }} \mathbb{E}\left[L\left(\boldsymbol{\Xi}_{1}\right)-L\left(\boldsymbol{\Xi}_{T+1}\right)\right], B=\frac{M S^{2}}{2 \eta_{T, \min }} \sum_{t=1}^{T}\left\|\boldsymbol{\eta}_{t}\right\|^{2}$. Then,
\begin{equation}
T^{1 / 2} \mathbb{E}[\min _{1 \leq t \leq T}\left\|\nabla L\left(\boldsymbol{\Xi}_{t}\right)\right\|] \leq \mathbb{E}[\sqrt{\Delta}] \leq \sqrt{A+B}=\mathcal{O}(1).
\end{equation}
In other words,  $\mathbb{E}[\underset{1 \leq t \leq T}{\min}\left\|\nabla L\left(\boldsymbol{\Xi}_{t}\right)\right\|] \leq T^{-1 / 2} \mathcal{O}(1) $, which concludes our proof.
\end{proof}
\end{document}